\documentclass[9pt]{extarticle}

\usepackage{spconf,amsmath,graphicx}
\usepackage{algorithm}
\usepackage{algpseudocode}
\usepackage{amssymb}
\usepackage{amsthm}
\usepackage{xcolor}
\usepackage{caption}
\usepackage{enumitem}
\usepackage{hyperref}
\newtheorem{theorem}{Theorem}
\newtheorem{lemma}{Lemma}
\newtheorem{proposition}{Proposition}

\makeatother
\algrenewcommand\algorithmicrequire{\textbf{In:}}
\algrenewcommand\algorithmicensure{\textbf{Out:}}
\algrenewcommand\algorithmicindent{0.8em}

\theoremstyle{definition}

\newtheorem{definition}[theorem]{Definition}

\newtheorem{remark}[theorem]{Remark}

\allowdisplaybreaks

\begin{document}

\title{Weighted Stochastic Differential Equation to Implement Wasserstein-Fisher-Rao Gradient Flow\vspace{-8pt}}
\markboth{Journal of \LaTeX\ Class Files,~Vol.~14, No.~8, August~2015}%
{Shell \MakeLowercase{\textit{et al.}}: Bare Demo of IEEEtran.cls for IEEE Journals}
\maketitle


\begin{abstract}
Score-based diffusion models currently constitute the state of the art in continuous generative modeling.
These methods are typically formulated via overdamped or underdamped Ornstein--Uhlenbeck–type stochastic differential equations, in which sampling is driven by a combination of deterministic drift and Brownian diffusion, resulting in continuous particle trajectories in the ambient space.
While such dynamics enjoy exponential convergence guarantees for strongly log-concave target distributions, it is well known that their mixing rates deteriorate exponentially in the presence of nonconvex or multimodal landscapes, such as double-well potentials.
Since many practical generative modeling tasks involve highly non-log-concave target distributions, considerable recent effort has been devoted to developing sampling schemes that improve exploration beyond classical diffusion dynamics.

A promising line of work leverages tools from information geometry to augment diffusion-based samplers with controlled mass reweighting mechanisms.
This perspective leads naturally to Wasserstein--Fisher--Rao (WFR) geometries, which couple transport in the sample space with vertical (reaction) dynamics on the space of probability measures.
In this work, we formulate such reweighting mechanisms through the introduction of explicit correction terms and show how they can be implemented via weighted stochastic differential equations using the Feynman--Kac representation.
Our study provides a preliminary but rigorous investigation of WFR-based sampling dynamics, and aims to clarify their geometric and operator-theoretic structure as a foundation for future theoretical and algorithmic developments. 

\end{abstract}

\section{Introduction}

Modern score-based diffusion models can be viewed as learning (or approximating) the
time-reversed dynamics of a forward noising diffusion, so that generation reduces to sampling by
simulating a stochastic process (or its probability-flow ODE counterpart); see, e.g.,
\cite{tang_zhao_survey_2025}.  
This “sampling-first” viewpoint makes the overall quality--compute tradeoff hinge on a classical
question in stochastic analysis and MCMC: \emph{how fast does a diffusion (and its discretization)
converge to its target distribution?}

When the target density has the form $\pi(\mathrm{d}x)\propto e^{-V(x)}\mathrm{d}x$ with
$V$ smooth and (strongly) convex, the associated Langevin dynamics enjoys \emph{quantitative}
convergence rates to equilibrium.
At the continuous-time level, a standard route is via functional inequalities
(Poincar\'e/log-Sobolev) and curvature-type criteria, notably the Bakry--\'Emery framework,
which yields exponential decay of suitable divergences under strong convexity/positive curvature
assumptions \cite{bakry_emery_1985,holley_stroock_1987,ledoux_markov_generators_2000}.  
At the algorithmic level, non-asymptotic guarantees for discretizations such as unadjusted
Langevin Monte Carlo (ULA/LMC) are by now well developed in the smooth log-concave setting,
including explicit dimension/accuracy scaling and robustness to gradient error
\cite{dalalyan_2017,dalalyan_karagulyan_2019,erdogdu_lmc_2020,erdogdu_lmc_chisq_2021}. 
Moreover, kinetic (underdamped) Langevin diffusions can yield improved complexity bounds over
overdamped LMC under comparable regularity/convexity conditions, and several analyses establish
accelerated rates in Wasserstein/KL-type metrics \cite{cheng_underdamped_2018,dalalyan_riou_durand_2020,shen_midpoint_2019}.  

The picture changes drastically once $V(\cdot)$ is \emph{nonconvex} and the target is multi-modal (e.g.,
double-well potentials or mixtures), which is the typical regime for many scientific and modern
generative-modeling tasks.
In such landscapes, Langevin-type dynamics may exhibit \emph{metastability}: trajectories spend
exponentially long times trapped near one mode before crossing an energy barrier to another.
Sharp asymptotics for transition times are classically captured by Kramers/Eyring--Kramers laws
and their refinements, which quantify the barrier-dominated nature of mixing in the low-noise
(or low-temperature) regime \cite{hanggi_talkner_borkovec_1990,eyring_kramers_underdamped_2025,lee_metastable_mixing_2025}.  
This explains, at a mechanism level, why ``exponential convergence'' in the log-concave case does
\emph{not} translate to effective sampling performance in multi-well settings: in the presence of
energy barriers, the spectral gap and log-Sobolev constant of the associated Langevin generator
typically decay exponentially in the barrier height, implying exponentially large global mixing
times \cite{holley_stroock_1987,helffer_nier_2005,bovier_hollander_2015,miclo_metastability_1995}.

A substantial modern literature addresses this slow-mixing obstruction by changing either
(i) the \emph{dynamics} (e.g., kinetic/nonreversible variants, couplings showing contraction under
weaker conditions) \cite{eberle_guillin_zimmer_2019},  
or (ii) the \emph{effective landscape} (e.g., tempering/replica-exchange ideas designed to move
between modes). As one representative example, simulated tempering combined with Langevin
updates yields provable improvements for certain mixtures of log-concave components
\cite{lee_simulated_tempering_2018}. 
These developments motivate the central theme of this paper: to systematically enrich the
sampling process by incorporating \emph{geometric} degrees of freedom (e.g., reweighting/Fisher--Rao
components in Wasserstein--Fisher--Rao-type formulations), aiming to mitigate metastability while
preserving a principled continuum description compatible with diffusion-model methodology.

\textbf{Brownian Motion is Not Enough to explore the space!} A central modeling choice in diffusion-based sampling (and in score-based generative modeling
via SDEs) is the \emph{driving noise}. In the classical overdamped Langevin diffusion,
\begin{equation}
  \mathrm{d}X_t = -\nabla f(X_t)\,\mathrm{d}t + \sqrt{2}\,\mathrm{d}B_t,
  \label{eq:overdamped-langevin}
\end{equation}
the Brownian term $B_t$ induces increments with typical size $\|B_t-B_s\|\asymp |t-s|^{1/2}$.
This $1/2$--scaling is intimately tied to Gaussianity, finite quadratic variation, and the
semimartingale structure that underpins It\^o calculus and the classical Fokker--Planck PDE.
However, in multimodal landscapes (e.g.\ double-well potentials), Brownian-driven dynamics can
become metastable: barrier crossing is dominated by rare fluctuations whose timescale is
exponentially large in the barrier height (cf.\ Kramers/Eyring--Kramers theory and related
metastability results) \cite{hanggi_talkner_borkovec_1990,bovier_kramers_2025,lee_metastable_2025}.
This motivates enriching the noise model to enable more effective exploration and faster
inter-mode transport.

A common informal heuristic is: \emph{Brownian paths explore locally (diffusively), whereas
heavy-tailed or jump-driven paths can occasionally relocate nonlocally}, potentially reducing
barrier-induced trapping. Two canonical generalizations of Brownian motion illustrate the tradeoff between (i) scaling
properties that can enhance exploration and (ii) the availability of It\^o calculus, and hence writing down the forward/backward equations:

\textbf{(A) Fractional Brownian motion (fBm)}: 
Fractional Brownian motion $(B^H_t)_{t\ge 0}$ with Hurst index $H\in(0,1)$ is a centered Gaussian
process with covariance
\begin{equation}
  \mathbb{E}\big[B^H_t B^H_s\big]
  = \tfrac12\big(t^{2H}+s^{2H}-|t-s|^{2H}\big),
  \qquad s,t\ge 0,
  \label{eq:fbm-cov}
\end{equation}
and self-similarity $B^H_{ct}\stackrel{d}{=}c^H B^H_t$ \cite{nualart_2006}.
For $H\neq \tfrac12$, fBm is \emph{not} a semimartingale, so the standard It\^o stochastic
integral (and It\^o formula) is not available; one must use alternative calculi (Young/rough
paths, Malliavin--Skorohod integrals, etc.) \cite{nualart_2006,jarrow_no_arbitrage_semimartingales_2009}.
In other words, fBm changes scaling (typical increments $\asymp t^H$) while remaining Gaussian,
but it breaks the semimartingale foundation of classical diffusion PDEs.

\textbf{(B) 
A general \emph{jump diffusion}} (or L\'evy-driven SDE) can be written as
\begin{multline}
  \mathrm{d}X_t= b(X_{t-})\,\mathrm{d}t
     + \sigma(X_{t-})\,\mathrm{d}B_t
     \\+ \int_{\|z\|<1} \gamma(X_{t-},z)\,\tilde{N}(\mathrm{d}t,\mathrm{d}z)
     \\+ \int_{\|z\|\ge 1} \Gamma(X_{t-},z)\,N(\mathrm{d}t,\mathrm{d}z),
  \label{eq:jump-sde}
\end{multline}
where $N$ is a Poisson random measure on $\mathbb{R}_+\times(\mathbb{R}^d\setminus\{0\})$ with
intensity $\mathrm{d}t\,\nu(\mathrm{d}z)$, and $\tilde{N}:=N-\mathrm{d}t\,\nu$ is the compensated
measure \cite{applebaum_2009,protter_2005}.
For sufficiently smooth test functions $\varphi$, one has an It\^o formula with jumps, producing
both differential and jump correction terms \cite{protter_2005,applebaum_2009}.

The infinitesimal generator $\mathcal{L}$ of a L\'evy-driven process is generally \emph{nonlocal}.
For a pure L\'evy process with triplet $(b,\Sigma,\nu)$, acting on smooth $\varphi$,
\begin{multline}
  \mathcal{L}\varphi(x)
  = \langle b,\nabla \varphi(x)\rangle
    + \tfrac12 \mathrm{Tr}\big(\Sigma \nabla^2 \varphi(x)\big)
    \\+ \int_{\mathbb{R}^d\setminus\{0\}}
        \Big(\varphi(x+z)-\varphi(x)-\langle \nabla\varphi(x),z\rangle\mathbf{1}_{\{\|z\|<1\}}\Big)\,\nu(\mathrm{d}z).
  \label{eq:levy-generator}
\end{multline}
When $\nu$ corresponds to an isotropic $\alpha$-stable process, $\mathcal{L}$ reduces (up to a
constant) to the fractional Laplacian $-(-\Delta)^{\alpha/2}$, highlighting the precise way in
which ``$t^{1/\alpha}$ scaling'' induces a \emph{fractional} (nonlocal) diffusion operator
\cite{sato_1999,applebaum_2009}.
Consequently, the forward equation for densities is a PIDE (often called a nonlocal
Fokker--Planck equation), rather than the classical second-order PDE associated with Brownian
noise. Explicit derivations and general forms of these equations for L\'evy-driven systems (and
related stochastic integrals) can be found in \cite{sun_li_zheng_2016,aip_fokker_planck_2012,sun_2017_pide}. In a multi-well landscape, the bottleneck is the probability of producing a fluctuation large
enough to cross the energy barrier. Brownian motion accomplishes this via many small steps(and potentially high rates of diffusion, i.e. higher temperature)
making crossing times exponentially sensitive to barrier height.
A jump component allows \emph{rare but macroscopic} moves: a single jump may relocate the state
across wells, potentially reducing metastable trapping. This comes at a price:
nonlocality in the generator and additional modeling/analysis complexity (e.g.\ stability of
discretizations, defining suitable likelihoods for training, and understanding how scores
transform under jump noise). Recent score-based generative modeling work explicitly explores
$\alpha$-stable L\'evy noise as the forward corruption process, precisely to leverage such heavy-tailed
jumps \cite{yoon_levy_score_2023}.

A central practical obstacle for \emph{non-Gaussian} noise designs (e.g.\ jump--diffusions, $\alpha$--stable drivers, or more general L\'evy processes) is that their simulation and analysis typically requires a Poisson random measure (or L\'evy--It\^o decomposition), nonlocal generators, and bespoke discretizations; moreover, many classical identities used in diffusion-model training (e.g.\ It\^o calculus, Girsanov transforms in their simplest form, and PDE duality with second-order elliptic operators) either fail outright or require substantial additional structure and technical overhead. This motivates a complementary route: \emph{retain Brownian motion at the particle level}, but enrich the \emph{law-level} dynamics by allowing \emph{mass reweighting} (and, when needed, resampling/branching). 
\begin{equation}
\mathrm{d}x_t = v_t(x_t)\mathrm{d}t + \sigma_t \mathrm{d}B_t,
\qquad
\mathrm{d}w_t = \bar \psi_t(x_t)\mathrm{d}t,
\end{equation}In other words, $x_t$ diffuses, and then we resample different paths w.r.t $w_t$, to allow teleportation.

The resulting Wasserstein--Fisher--Rao (WFR) or Hellinger--Kantorovich (HK) metric equips the space
of positive measures with a genuine Riemannian structure whose tangent vectors decompose into
horizontal (transport) and vertical (reaction) components \cite{liero_hk_2016,liero_optimal_2018,chizat_unbalanced_2018}.
Gradient flows in this geometry rigorously capture the coupled dynamics
\begin{equation}
  \partial_t \rho_t
  = -\nabla\!\cdot(\rho_t v_t) + \rho_t \overline{\psi}_t,
\end{equation}
which subsumes weighted Fokker--Planck equations, birth--death samplers, and normalized
Feynman--Kac flows as special cases.

Recent work has begun to extend curvature-based analysis (à la Bakry--\'Emery or Lott--Sturm--Villani)
to unbalanced and information-geometric settings.
Lower curvature bounds in WFR/HK spaces yield contractivity and stability results for coupled
transport--reaction flows, paralleling classical Wasserstein theory but with additional
degrees of freedom \cite{mondino_unbalanced_curvature_2023,erbar_maas_2012}.
These results provide a geometric explanation for why adding Fisher--Rao components can improve
mixing: the effective curvature of the sampling manifold is altered, potentially enlarging
spectral gaps and reducing geodesic distances between modes.

The perspective adopted in this work is that, while \emph{genuine} jump-driven or non-Gaussian noise
generalizations can be powerful tools for enhancing exploration, they often entail a substantial
burden both in stochastic calculus and in practical implementation.
Weighted stochastic differential equations and their interacting particle system (IPS) realizations,
by contrast, provide a pragmatic intermediate approach: they preserve diffusive particle dynamics
and thus remain compatible with standard It\^o calculus and numerical SDE methods, while enriching
the evolution at the \emph{law level} through a controllable vertical (reaction) component.
This vertical component admits a natural interpretation in terms of Fisher--Rao geometry and is
intrinsically compatible with the Wasserstein--Fisher--Rao (WFR), or equivalently
Hellinger--Kantorovich (HK), framework, including branching and resampling interpretations.

The organization of the paper reflects this viewpoint.
In Section~\ref{sec:geometry}, we formally introduce the Riemannian geometry of the
Wasserstein--Fisher--Rao space and review the essential elements of its differential structure.
Having established the underlying geometric framework, Section~\ref{sec:pde to sde} shows that
weighted stochastic differential equations induce the corresponding WFR evolution at the level of
probability densities, thereby yielding an implementable sampling scheme that realizes
Fisher--Rao-type geometric corrections.
The proof relies on a Feynman--Kac representation and follows established arguments in the literature.
We subsequently present an equivalent formulation in terms of jump or branching processes, further
clarifying the relationship between weighted diffusions and nonlocal dynamics.

Both the Feynman--Kac correspondence between weighted SDEs and WFR-type partial differential
equations, as well as the equivalence with jump-process formulations, are classical results and can
be found in the cited references; accordingly, they are \textbf{not claimed as novel contributions}
of this work\footnote{On the other hand, there is an immediate application: if you want to avoid retraining diffusion models for $q^1, q^2$, is there a way to use them to sample from different mixtures of $q^i$? This, in particular for information preserving mixtures, is answered in the last section of Appendix}.

 \section{How to build the geometries?}
\label{sec:geometry}
 Throughout, we denote by $\mathcal{P}(\mathbb{R}^d)$ the set of Borel probability measures on $\mathbb{R}^d$, and by
\[
\mathcal{P}^2_{\mathrm{ac}}(\mathbb{R}^d)
    \;:=\;
    \Bigl\{
        \rho \in \mathcal{P}(\mathbb{R}^d)
        \;:\;
        \rho \ll \mathsf{Leb},
        \quad
        \int_{\mathbb{R}^d} \|x\|^2 \,\rho(x)\,\mathrm{d}x < \infty
    \Bigr\},
\]
the space of absolutely continuous probability measures with finite second moment.
Endowing $\mathcal{P}^2_{\mathrm{ac}}(\mathbb{R}^d)$ with additional geometric structure allows one to define
tangent spaces, variational and gradient flows, geodesics, and curvature-type notions, thereby enabling a
differential-geometric treatment of evolution equations on the space of probability measures.Unlike finite-dimensional Euclidean spaces, however, there is no unique canonical geometry on
$\mathcal{P}^2_{\mathrm{ac}}(\mathbb{R}^d)$.
Instead, several distinct—yet deeply interconnected—geometric frameworks arise, depending on which
fundamental notions of classical geometry, one seeks to generalize.

At a conceptual level, geometric structures on abstract (possibly infinite-dimensional) manifolds may be
introduced by generalizing one or more of the following fundamental aspects of Euclidean geometry:
\begin{enumerate}
    \item \emph{Metric structure:} generalizing straight lines as length-minimizing curves, leading to
    metric or length spaces and geodesic distances;
    \item \emph{Curvature structure:} generalizing angles and second-order variation, allowing one to
    quantify convexity, contraction, and curvature bounds;
    \item \emph{Affine structure:} generalizing straight lines as curves of minimal acceleration, leading
    to affine connections, parallel transport, and families of distinguished geodesics.
\end{enumerate}
Each of these perspectives yields a different—but mathematically natural—geometry on the space of
probability measures.
In the sequel, we show how these three viewpoints give rise, respectively, to Wasserstein geometry,
Fisher--Rao information geometry, and connection-based (affine) structures, and how their interaction
underpins the Wasserstein--Fisher--Rao framework studied in this work.
\textbf{(i) Geometry from a distance or divergence.}
A classical and conceptually transparent approach begins with a notion of ``distance'' $D(\mu,\nu)$ between two probability measures.  
If $D$ is a metric, or more generally a divergence with suitable convexity and lower semicontinuity properties, one can define geodesics as curves $\gamma : [0,1] \to \mathcal{P}(\mathbb{R}^d)$ minimizing the length functional
\[
L(\gamma)
    = \sup_{0=t_0<t_1<\cdots<t_N=1} \sum_{i=0}^{N-1} 
    D\bigl(\gamma(t_i),\gamma(t_{i+1})\bigr).
\]
This viewpoint is inherently variational: the geometry is encoded in the minimizing properties of curves.  
The Wasserstein space $(\mathcal{P}^2(\mathbb{R}^d), W_2)$ is the canonical example.  
Geodesics are displacement interpolations, obtained by pushing forward a measure along optimal transport maps; their variational characterization leads directly to the Otto calculus and second-order differential structure\cite{Villani2009Optimal}.  

Another family of geometries arises when $D$ is a statistical divergence, such as the Kullback--Leibler divergence, squared Hellinger distance, or the Amari $\alpha$-divergences \cite{amari_information_2016}.  
In this setting, geodesics correspond to ``interpolations induced by the divergence'', and need not have a transport interpretation.  
For example, the Fisher--Rao geodesic between densities $\rho_0$ and $\rho_1$ is given by the explicit formula
\[
\rho_t = \bigl( (1-t)\sqrt{\rho_0} + t \sqrt{\rho_1} \bigr)^2,
\]
which is the geodesic associated with the squared Hellinger distance.  
In these divergence-based geometries, curvature, convexity, and stability properties follow from the analytic structure of the divergence.

\medskip

\textbf{(ii) Geometry from a Riemannian metric (angle-based viewpoint).}
A more differential perspective begins by postulating a Riemannian metric on each tangent space $T_\rho \mathcal{P}^2_{\mathrm{ac}}$, that is, an inner product(more formally a (0,2)-tensor, or a 2 form)
\[
\langle \xi, \eta \rangle_\rho, \qquad 
    \xi,\eta \in T_\rho \mathcal{P}^2_{\mathrm{ac}}.
\]
This requires specifying a model for the tangent space.  
In Wasserstein geometry, the tangent vectors are velocity fields $v$ satisfying the continuity equation 
$\partial_t \rho + \nabla \cdot (\rho v) = 0$, and the Riemannian metric is
\[
\langle v_1, v_2\rangle_\rho
    := \int_{\mathbb{R}^d} \rho(x) \, v_1(x)\cdot v_2(x)\,dx.
\]
This makes $\mathcal{P}^2_{\mathrm{ac}}$ into a formal infinite-dimensional Riemannian manifold where gradient flows of functionals become PDEs.  
For example, the Fokker--Planck equation is the Wasserstein gradient flow of the free energy functional \cite{ambrosio_gradient_flows_2008}.

In Fisher--Rao geometry, the tangent vectors take the form 
$\xi = \dot{\rho}$ with $\int \xi = 0$ and the metric reads
\[
\langle \xi,\eta \rangle_\rho := 
    \int_{\mathbb{R}^d} \frac{\xi(x)\,\eta(x)}{\rho(x)}\,dx,
\]
inducing a distinct Riemannian structure not based on transport but on multiplicative perturbations.  
This construction is the infinite-dimensional limit of the classical Fisher information metric on statistical models \cite{amari_information_2016}.  
Unlike the Wasserstein metric, the Fisher--Rao metric treats mass change, not displacement, as the fundamental mode of variation.

The Riemannian viewpoint provides notions of angles, lengths of curves, Levi--Civita connections, and curvature tensors.  
It supports a second-order differential calculus on $\mathcal{P}^2_{\mathrm{ac}}$ that is indispensable for stability analysis, Bochner-type formulas, and geometric interpretations of PDEs.

\medskip

\textbf{(iii) Geometry from affine connections and dualistic structures.}
A third pathway, prominent in information geometry, develops geometry starting from an affine connection $\nabla$ rather than a metric. This notion of building geometry tries to generalize the concept of straightness of lines(henceforth geodesics), rather than the length.  
This connection determines parallel transport, geodesics, and curvature, and need not arise from any Riemannian metric.  
In the dualistic formalism of Amari \cite{amari_information_2016}, one introduces a pair of torsion-free affine connections $(\nabla,\nabla^*)$ that are dual with respect to a Riemannian metric $g$.  
Different choices of $\nabla$ lead to different families of geodesics:  
mixture geodesics, exponential geodesics, and more generally the $\alpha$-geodesics associated with the $\alpha$-connections.  
These geodesics admit simple coordinate expressions:  
for instance, mixture geodesics are straight lines in the space of densities, while exponential geodesics are straight lines in logarithmic coordinates.  
In this viewpoint, geodesic structure is primary, and the Riemannian metric appears only as a dualizing object relating $\nabla$ and $\nabla^*$.

This affine-geometric construction generalizes naturally to infinite-dimensional spaces of measures.  
For instance, $\alpha$-connections induce ``midpoint operators'' defining interpolations between measures that are not derived from minimizing a distance but from preserving affine structure in a chosen coordinate system.  
These structures play a central role in dual-flat geometries and provide the foundation for Bregman divergences, convex potentials, and natural gradient methods.

\medskip

\textbf{Summary.}
The three approaches above are not competing theories but complementary ones. Starting from one viewpoint toward the geometry, one can build other concepts as well, though it is important to have them in our arsenal to use these tools to better understand the geometries

These frameworks collectively demonstrate that the space of probability measures supports a rich family of geometric structures.  
Choosing one determines the analytic form of PDEs (as gradient flows or geodesic equations), the SDEs representing their particle dynamics, and the curvature properties governing stability and convergence.  
The remainder of this work builds upon these foundational constructions to examine Wasserstein geometry, Fisher--Rao geometry, and their hybridizations in detail.
\subsection{Wasserstein Geometry, Otto Calculus}

For $\mu,\nu \in \mathcal{P}_2(\Omega)$, the $2$-Wasserstein distance is defined by
\[
W_2^2(\mu,\nu)
:= \inf_{\pi \in \Pi(\mu,\nu)} \int_{\Omega \times \Omega} \|x-y\|^2 \,\mathrm{d}\pi(x,y),
\]
where $\Pi(\mu,\nu)$ is the set of couplings of $\mu$ and $\nu$. Equipped with $W_2$, the space $\big(\mathcal{P}_2(\Omega),W_2\big)$ is a complete separable metric space; in particular, one can speak of absolutely continuous curves, metric derivatives, and so on.

\begin{definition}[Metric derivative]
Let $(X,d)$ be a metric space and let $(x_t)_{t\in[0,1]}$ be a curve in $X$. The \emph{metric derivative} of $(x_t)$ at time $t$ (when it exists) is
\[
|x'|(t) := \lim_{s\to t} \frac{d(x_s,x_t)}{|s-t|}.
\]
A curve is called absolutely continuous if there exists $m\in L^1([0,1])$ such that
\[
d(x_s,x_t) \leq \int_s^t m(r)\,\mathrm{d}r,\qquad 0 \leq s \leq t \leq 1.
\]
\end{definition}

We are primarily interested in absolutely continuous curves $(\mu_t)_{t\in[0,1]}$ in $(\mathcal{P}_2(\Omega),W_2)$. A basic result (Ambrosio--Gigli--Savaré) says that such curves admit an \emph{Eulerian} description via a continuity equation.

\begin{definition}[Continuity equation]
A narrowly continuous curve $(\mu_t)_{t\in[0,1]} \subset \mathcal{P}_2(\Omega)$ satisfies the continuity equation with (time-dependent) velocity field $(v_t)_{t\in[0,1]}$ if
\begin{equation}
\partial_t \mu_t + \nabla \!\cdot (\mu_t v_t) = 0
\label{eq:continuity}
\end{equation}
in the sense of distributions, i.e.\ for every $\varphi \in C_c^\infty(\Omega)$
\[
\frac{\mathrm{d}}{\mathrm{d}t} \int_\Omega \varphi \,\mathrm{d}\mu_t
= \int_\Omega \nabla \varphi(x) \cdot v_t(x) \,\mathrm{d}\mu_t(x)
\quad \text{for a.e. } t.
\]
\end{definition}

Formally, one may think of $(\mu_t)$ as the distribution of a random particle $X_t$ satisfying the ODE
\[
\dot X_t = v_t(X_t), \qquad X_0 \sim \mu_0,
\]
and then $\mu_t = (\mathrm{Law}\,X_t)$ solves \eqref{eq:continuity}. The pair $(\mu_t,v_t)$ provides an Eulerian (density/velocity) description, while $X_t$ gives a Lagrangian (particle/trajectory) description.

A crucial fact is that, for a given curve $(\mu_t)$, the velocity field $v_t$ solving \eqref{eq:continuity} is not unique: one may add divergence-free fields $\tilde v_t$ with $\nabla\!\cdot(\mu_t \tilde v_t)=0$ without changing the evolution of $\mu_t$. The standard resolution is to select, at each time, the \emph{minimal kinetic energy} representative.

Define the kinetic energy of $(\mu_t,v_t)$ on $[0,1]$ by
\[
\int_0^1 \|v_t\|^2_{\mu_t}\,\mathrm{d}t
:= \int_0^1 \int_\Omega \|v_t(x)\|^2\,\mathrm{d}\mu_t(x)\,\mathrm{d}t.
\]
There is a deep link between this functional and the Wasserstein metric: the Benamou--Brenier formula states that
\begin{equation}
W_2^2(\mu_0,\mu_1)
=
\inf_{(\mu_t,v_t)} \int_0^1 \int_\Omega \|v_t(x)\|^2 \,\mathrm{d}\mu_t(x)\,\mathrm{d}t,
\label{eq:benamou-brenier}
\end{equation}
where the infimum runs over all narrowly continuous $(\mu_t)_{t\in[0,1]}$ connecting $\mu_0$ to $\mu_1$ and all Borel vector fields $(v_t)$ solving \eqref{eq:continuity}. Minimizers $(\mu_t,v_t)$ correspond to constant-speed geodesics in $(\mathcal{P}_2(\Omega),W_2)$.

This suggests a Riemannian interpretation: at each $\mu$, we declare the \emph{tangent space} to be the closure of gradient vector fields in $L^2(\mu)$.

\begin{definition}[Wasserstein tangent space and metric]
Let $\mu \in \mathcal{P}_2(\Omega)$ admit a smooth strictly positive density with respect to Lebesgue, which we still denote by $\mu(x)$. The (formal) tangent space at $\mu$ is
\[
T_\mu \mathcal{P}_2(\Omega)
:= \overline{\{\nabla \phi : \phi \in C_c^\infty(\Omega)\}}^{\,L^2(\mu)},
\]
and the Riemannian metric at $\mu$ is given by the $L^2(\mu)$ inner product
\[
\langle \nabla \phi_1, \nabla \phi_2 \rangle_\mu
:= \int_\Omega \nabla \phi_1(x)\cdot \nabla \phi_2(x)\,\mathrm{d}\mu(x).
\]
We write $\|v\|_\mu^2 := \langle v,v\rangle_\mu$.
\end{definition}

Under this identification, the minimal-kinetic-energy velocity field $v_t$ associated with a curve $(\mu_t)$ is uniquely characterized (up to $\mu_t$-null sets), satisfies $v_t \in T_{\mu_t}\mathcal{P}_2(\Omega)$, and its norm coincides with the metric derivative:
\[
|{\mu}'|(t) = \|v_t\|_{\mu_t} \quad \text{for a.e. } t.
\]
At the level of optimal transport maps, if $T_{\mu_t\to\mu_{t+h}}$ denotes the optimal map from $\mu_t$ to $\mu_{t+h}$ and $\mathrm{id}$ is the identity map, then formally
\[
v_t = \lim_{h\to 0}\frac{T_{\mu_t\to\mu_{t+h}} - \mathrm{id}}{h}
\quad\text{in }L^2(\mu_t).
\]

Let $F\colon \mathcal{P}_2(\Omega)\to \mathbb{R}$ be a functional. Its \emph{first variation} at $\mu$ (when it exists) is a scalar function $\frac{\delta F}{\delta \mu}(\mu)\colon \Omega\to \mathbb{R}$, defined (up to an additive constant) by
\[
\left.\frac{\mathrm{d}}{\mathrm{d}\varepsilon}\right|_{\varepsilon=0}
F(\mu + \varepsilon \sigma)
=
\int_\Omega \frac{\delta F}{\delta \mu}(\mu)(x) \,\mathrm{d}\sigma(x)
\]
for all signed measures $\sigma$ such that $\mu + \varepsilon\sigma \in \mathcal{P}_2(\Omega)$ for $\varepsilon$ small enough.If $(\mu_t)$ is a curve with velocity field $v_t$ satisfying the continuity equation, then
\begin{multline*}
\frac{\mathrm{d}}{\mathrm{d}t}F(\mu_t)
= \int_\Omega \frac{\delta F}{\delta \mu}(\mu_t)(x)\,
\partial_t \mu_t(x)\,\mathrm{d}x
\\= -\int_\Omega \frac{\delta F}{\delta \mu}(\mu_t)(x)\,
\nabla\!\cdot\big(\mu_t(x)v_t(x)\big)\,\mathrm{d}x,
\end{multline*}
and an integration by parts yields
\[
\frac{\mathrm{d}}{\mathrm{d}t}F(\mu_t)
= \int_\Omega \nabla \!\left(\frac{\delta F}{\delta \mu}(\mu_t)\right)(x)
\cdot v_t(x)\,\mathrm{d}\mu_t(x)
= \langle \nabla \tfrac{\delta F}{\delta \mu}(\mu_t),\,v_t\rangle_{\mu_t}.
\]

\begin{definition}[Wasserstein gradient]
The (formal) Wasserstein gradient of $F$ at $\mu$ is the element
\[
\nabla^{W_2} F(\mu) \in T_\mu \mathcal{P}_2(\Omega)
\]
characterized by
\[
\frac{\mathrm{d}}{\mathrm{d}t}F(\mu_t)
= \big\langle \nabla^{W_2}F(\mu_t), v_t\big\rangle_{\mu_t}
\quad\text{for all curves }(\mu_t)\text{ with velocity }v_t.
\]
From the computation above one obtains
\[
\nabla^{W_2}F(\mu) = \nabla \!\left(\frac{\delta F}{\delta \mu}(\mu)\right).
\]
\end{definition}

The gradient flow equation associated with $F$ in the Wasserstein geometry is obtained by following the steepest descent direction:
\[
v_t = - \nabla^{W_2}F(\mu_t).
\]
Plugging into the continuity equation yields the PDE
\begin{equation}
\partial_t \mu_t = \nabla\!\cdot\!\left(
\mu_t\, \nabla \!\left(\frac{\delta F}{\delta \mu}(\mu_t)\right)
\right).
\label{eq:wasserstein-gradient-flow}
\end{equation}
Formally,
\[
\frac{\mathrm{d}}{\mathrm{d}t}F(\mu_t)
= - \big\|\nabla^{W_2}F(\mu_t)\big\|_{\mu_t}^2 \leq 0,
\]
so $F(\mu_t)$ is non-increasing along the flow.

A central sufficient condition for quantitative convergence is (geodesic) $\lambda$-convexity of $F$ in $(\mathcal{P}_2(\Omega),W_2)$: if $F$ is $\lambda$-geodesically convex, then the gradient flow is well-posed and $F(\mu_t)-\inf F$ decays at least exponentially at rate $2\lambda$. There are deep connections between $W_2$ being geodesically convex and $\mathcal{P}^W_{ac}$ being non-negatively curved as a length space(for more information, check \cite{chewi_statistical_2024}). Before we continue to the next geoemtry, it is worth mentioning briefly about the other concepts, other than the metric, in this Riemannian Manifold. Though not explicitly mentioned in most of the literature, one can construct the connection and covariant derivative over this space, and talk about the Ricci Curvature. But as this is an infinite dimensional Riemannian Manifold(hence called psudo Riemannian), the derivations are not as clean. That's the reason that talking about curvature is much easier if we look into these pdes from a ponit of view of operator theory. We refer the reader to appendix for more details of the Wasserstein geometry, including curvature, connection, and geodesics concepts.


\subsection{Information Geometry}

We now describe the Fisher--Rao metric on the space of positive measures. It is most convenient to start from the embedding $\mu \mapsto \sqrt{\mu}$ into $L^2(\Omega)$. Though it has first been introduced for parametric and hence finite-dimensional probability measures, we refer the reader to \cite{amari_information_2016} for the parametric approach. Here, however, we try to talk more about the infinite-dimensional case. 

Assume $\mu_0,\mu_1 \in \mathcal{M}_+(\Omega)$ admit densities (still denoted by $\mu_i$) with respect to Lebesgue. The Fisher--Rao distance is
\[
d^2_{\mathrm{FR}}(\mu_0,\mu_1)
:= \int_\Omega \big(\sqrt{\mu_0(x)} - \sqrt{\mu_1(x)}\big)^2\,\mathrm{d}x
= \big\|\sqrt{\mu_0}-\sqrt{\mu_1}\big\|_{L^2(\Omega)}^2.
\]
Thus the mapping $\mu\mapsto \sqrt{\mu}$ is an isometric embedding of $(\mathcal{M}_+(\Omega),d_{\mathrm{FR}})$ into the Hilbert space $L^2(\Omega)$; geometrically, $(\mathcal{M}_+(\Omega),d_{\mathrm{FR}})$ is a (flat) cone.

Consider a smooth curve $(\mu_t)_{t\in [0,1]} \subset \mathcal{M}_+(\Omega)$ with density $\mu_t(x)$ and time derivative $\partial_t \mu_t$. Differentiating $\sqrt{\mu_t}$ yields
\[
\partial_t \sqrt{\mu_t} = \frac{1}{2}\,\frac{\partial_t \mu_t}{\sqrt{\mu_t}},
\]
and the squared speed in $L^2$ is
\[
\big\|\partial_t \sqrt{\mu_t}\big\|_{L^2}^2
= \int_\Omega \frac{(\partial_t \mu_t)^2}{4\mu_t} \,\mathrm{d}x.
\]
This suggests the following Riemannian metric on $\mathcal{M}_+(\Omega)$.

\begin{definition}[Fisher--Rao metric on $\mathcal{M}_+(\Omega)$]
The (formal) tangent space at $\mu \in \mathcal{M}_+(\Omega)$ is
\[
T_\mu \mathcal{M}_+(\Omega)
:= \{ \dot \mu \in L^2_{\mathrm{loc}}(\Omega) : \dot \mu \text{ is a signed density} \},
\]
with Riemannian metric
\[
g_\mu^{\mathrm{FR}}(\dot\mu,\dot\mu)
:= \int_\Omega \frac{\dot\mu(x)^2}{4\,\mu(x)}\,\mathrm{d}x
\]
whenever the right-hand side is finite.
\end{definition}

It is often convenient to reparametrize tangent vectors multiplicatively. Given a curve $(\mu_t)$, define the \emph{reaction rate}
\[
\varphi_t(x) := \frac{\partial_t \mu_t(x)}{\mu_t(x)}.
\]
Then $\partial_t\mu_t = \varphi_t \mu_t$, and the Fisher--Rao metric reads
\begin{multline*}
g_\mu^{\mathrm{FR}}(\varphi,\varphi)
= g_\mu^{\mathrm{FR}}(\varphi\mu,\varphi\mu)
= \int_\Omega \frac{\varphi(x)^2 \mu(x)^2}{4\mu(x)}\,\mathrm{d}x
\\= \frac{1}{4} \int_\Omega \varphi(x)^2\,\mathrm{d}\mu(x).
\end{multline*}
Up to an inessential constant factor, we may thus identify the tangent space with functions $\varphi$ and use the inner product
\[
\langle \varphi_1,\varphi_2\rangle_\mu^{\mathrm{FR}}
:= \int_\Omega \varphi_1(x)\,\varphi_2(x)\,\mathrm{d}\mu(x).
\]

For \emph{probability} measures, i.e.\ $\mu\in\mathcal{P}(\Omega)$, we additionally impose mass conservation $\int_\Omega \partial_t\mu_t = 0$, which translates into $\int_\Omega \varphi_t\,\mathrm{d}\mu_t = 0$. Thus
\[
T_\mu \mathcal{P}(\Omega)
:= \left\{ \varphi\in L^2(\mu): \int_\Omega \varphi\,\mathrm{d}\mu = 0\right\},
\]
equipped with the same inner product.Let $F\colon \mathcal{M}_+(\Omega)\to\mathbb{R}$ be a functional with first variation $\frac{\delta F}{\delta\mu}(\mu)$. If $(\mu_t)$ is a curve with $\partial_t\mu_t = \varphi_t \mu_t$, then
\[
\frac{\mathrm{d}}{\mathrm{d}t}F(\mu_t)
= \int_\Omega \frac{\delta F}{\delta \mu}(\mu_t)\,\partial_t\mu_t
= \int_\Omega \frac{\delta F}{\delta \mu}(\mu_t)\,\varphi_t\,\mathrm{d}\mu_t.
\]
Comparing with the Fisher--Rao inner product, we see that in the unbalanced case
\[
\nabla^{\mathrm{FR}} F(\mu) = \frac{\delta F}{\delta\mu}(\mu),
\]
as an element of $T_\mu\mathcal{M}_+(\Omega)$ parametrized by $\varphi$. In the probabilistic (balanced) setting we must project onto the mean-zero subspace:
\[
\nabla^{\mathrm{FR}} F(\mu) = 
\frac{\delta F}{\delta\mu}(\mu) - \int_\Omega \frac{\delta F}{\delta\mu}(\mu)\,\mathrm{d}\mu.
\]

The Fisher--Rao gradient flow therefore takes the form
\begin{equation}
\partial_t \mu_t = - \nabla^{\mathrm{FR}} F(\mu_t)\,\mu_t.
\label{eq:fr-flow}
\end{equation}
In the unbalanced case this reduces to $\partial_t \mu_t = -\frac{\delta F}{\delta\mu}(\mu_t)\,\mu_t$, i.e.\ a pointwise exponential decay or growth driven by the first variation. In the balanced case, we subtract the spatial mean to keep the total mass equal to one.

\subsection{Hybrid Geometry, Hellinger--Kantorovich, Wasserstein--Fisher--Rao}

The Wasserstein geometry models pure transport of mass; Fisher--Rao models pure creation/annihilation (or, in the probability case, reweighting). In many applications, both mechanisms occur simultaneously: mass is transported in space and its intensity changes in time. The WFR (or Hellinger--Kantorovich) geometry is a natural Riemannian structure on measures that combines both effects in a single metric framework.

Let $(\mu_t)_{t\in[0,1]}$ be a curve in $\mathcal{M}_+(\Omega)$, the cone of finite positive measures on $\Omega$. We now consider pairs $(\varphi_t,v_t)$, where
\begin{equation}
\varphi_t\colon\Omega\to\mathbb{R}\quad
v_t\colon\Omega\to\mathbb{R}^d
\end{equation}
coupled through the \emph{continuity equation with reaction}
\begin{equation}
\partial_t \mu_t + \nabla\!\cdot(\mu_t v_t) = \varphi_t \mu_t.
\label{eq:wfr-continuity}
\end{equation}
When $\varphi_t\equiv 0$ this reduces to pure transport (the continuity equation of Wasserstein geometry); when $v_t\equiv 0$ it reduces to the Fisher--Rao evolution (pure reaction).

In analogy with the Benamou--Brenier formulation \eqref{eq:benamou-brenier}, one defines a kinetic action functional combining both transport and reaction costs:
\[
\mathcal{A}\big((\mu_t,\varphi_t,v_t)_{t\in[0,1]}\big)
:= \int_0^1 \!\!\int_\Omega \big( \|v_t(x)\|^2 + \varphi_t(x)^2 \big)\,\mathrm{d}\mu_t(x)\,\mathrm{d}t.
\]
The WFR (Hellinger--Kantorovich) distance $d_{\mathrm{WFR}}$ between two measures is obtained by minimizing this action over all triples $(\mu_t,\varphi_t,v_t)$ connecting $\mu_0$ to $\mu_1$ and satisfying \eqref{eq:wfr-continuity}:
\begin{multline*}
d_{\mathrm{WFR}}^2(\mu_0,\mu_1)
:= \inf\Big\{\mathcal{A}\big((\mu_t,\varphi_t,v_t)_{t\in[0,1]}\big):
(\mu_t,\varphi_t,v_t)\ \\ \text{satisfy \eqref{eq:wfr-continuity},}\ \mu_{t=0}=\mu_0,\ \mu_{t=1}=\mu_1\Big\}.
\end{multline*}
This dynamic formulation yields a metric on $\mathcal{M}_+(\Omega)$ that coincides with $W_2$ when only transport is allowed (i.e.\ $\varphi_t \equiv 0$ is enforced), and with $d_{\mathrm{FR}}$ when only reaction is allowed (i.e.\ $v_t\equiv 0$ is enforced). Static (Kantorovich-type) formulations and cone representations are also available in the literature on unbalanced optimal transport, but we emphasize here the Riemannian differential structure suggested by the dynamic formulation.

As in the Wasserstein setting, the pair $(\varphi_t,v_t)$ satisfying \eqref{eq:wfr-continuity} is not unique. There is a gauge freedom (adding divergence-free $v_t$ and adjusting $\varphi_t$ appropriately) that leaves $\partial_t\mu_t$ unchanged. One can again show that there is a unique “minimal norm” representative in the sense of the action $\mathcal{A}$, and that it can be parametrized by a scalar potential $\phi_t\colon\Omega\to\mathbb{R}$, with
\[
\varphi_t = \phi_t, \qquad v_t = \nabla \phi_t.
\]
Formally, this identifies the tangent directions at $\mu$ with potentials $\phi$ modulo $\mu$-almost everywhere constants, in complete analogy with the Otto calculus where gradients $\nabla\phi$ represent transport directions.

This leads to the following (formal) tangent space (see, for instance, Chapters~6 and~7 of \cite{chewi_statistical_2024}):

\begin{definition}[Wasserstein--Fisher--Rao tangent space]
For $\mu\in\mathcal{M}_+(\Omega)$, the formal WFR tangent space is
\[
T_\mu^{\mathrm{WFR}} \mathcal{M}_+(\Omega)
:= \overline{\big\{ (\phi,\nabla\phi): \phi\in C_c^\infty(\Omega)\big\}}^{\,L^2(\mu)},
\]
equipped with the inner product
\begin{multline}
\big\langle (\phi_1,\nabla\phi_1),(\phi_2,\nabla\phi_2)\big\rangle_\mu^{\mathrm{WFR}}
\\:= \int_\Omega \left(\phi_1(x)\phi_2(x) + \nabla\phi_1(x)\cdot\nabla\phi_2(x)\right)\,\mathrm{d}\mu(x),
\end{multline}
and associated norm
\[
\|(\phi,\nabla\phi)\|_\mu^{2}
= \int_\Omega \big( \phi(x)^2 + \|\nabla\phi(x)\|^2 \big)\,\mathrm{d}\mu(x).
\]
\end{definition}

Thus the WFR metric completes the Wasserstein metric $\|\nabla\phi\|_{L^2(\mu)}$ by adding the zeroth-order term $\|\phi\|_{L^2(\mu)}$, i.e.\ the full $H^1$-type norm of the potential $\phi$ with respect to $\mu$. From the Riemannian viewpoint, this is precisely the inner product induced by the operator
\[
\phi \mapsto (I - \Delta_\mu)\phi,
\qquad
\Delta_\mu := \nabla\!\cdot(\mu \nabla),
\]
on potentials, so that the WFR geometry can be seen as an $H^1$-Sobolev deformation of the $L^2$-based Wasserstein geometry.

Restricting to probability measures $\mathcal{P}(\Omega)\subset\mathcal{M}_+(\Omega)$, one imposes the additional constraint $\int_\Omega \phi\,\mathrm{d}\mu=0$ to ensure conservation of total mass (so that $\int_\Omega \varphi_t\,\mathrm{d}\mu_t=0$).

Let $F\colon \mathcal{M}_+(\Omega)\to\mathbb{R}$ have first variation $\frac{\delta F}{\delta\mu}(\mu)$. Consider a curve $(\mu_t)$ with tangent $(\varphi_t,v_t)$ satisfying \eqref{eq:wfr-continuity}. We compute
\begin{align*}
\frac{\mathrm{d}}{\mathrm{d}t}F(\mu_t)
&= \int_\Omega \frac{\delta F}{\delta \mu}(\mu_t)\,\partial_t\mu_t\,\mathrm{d}x \\
&= \int_\Omega \frac{\delta F}{\delta\mu}(\mu_t)\,\big(-\nabla\!\cdot(\mu_t v_t) + \varphi_t\mu_t\big)\,\mathrm{d}x \\
&= \int_\Omega \nabla\!\left(\frac{\delta F}{\delta\mu}(\mu_t)\right)\!\cdot v_t\,\mathrm{d}\mu_t
   + \int_\Omega \frac{\delta F}{\delta\mu}(\mu_t)\,\varphi_t\,\mathrm{d}\mu_t,
\end{align*}
where we integrated by parts in the transport term. If we parametrize $(\varphi_t,v_t)$ by $(\phi_t,\nabla\phi_t)$, then the right-hand side becomes
\[
\int_\Omega \left(
\frac{\delta F}{\delta\mu}(\mu_t)\,\phi_t
+ \nabla\!\left(\frac{\delta F}{\delta\mu}(\mu_t)\right)\!\cdot \nabla\phi_t
\right)\mathrm{d}\mu_t,
\]
which we recognize as the dual pairing with
\[
\left(
\frac{\delta F}{\delta\mu}(\mu_t),\,
\nabla\!\left(\frac{\delta F}{\delta\mu}(\mu_t)\right)
\right)
\in T_{\mu_t}^{\mathrm{WFR}}\mathcal{M}_+(\Omega).
\]

\begin{definition}[WFR gradient]
The (formal) WFR gradient of $F$ at $\mu$ is the element
\[
\nabla^{\mathrm{WFR}} F(\mu)
=
\left(
\frac{\delta F}{\delta\mu}(\mu),\,
\nabla\!\left(\frac{\delta F}{\delta\mu}(\mu)\right)
\right)
\in T_\mu^{\mathrm{WFR}}\mathcal{M}_+(\Omega).
\]
\end{definition}

\begin{table*}[h!]
\centering
\renewcommand{\arraystretch}{1.6}
\begin{tabular}{|c|c|c|}
\hline
\textbf{Geometry} & \textbf{Tangent Space} $T_\mu$ & \textbf{Metric} $g_\mu$ \\
\hline
Wasserstein ($W_2$) 
&
$\displaystyle T_\mu = \overline{\{\nabla\phi\}}^{L^2(\mu)}$
&
$\displaystyle g_\mu(v,w)=\int_\Omega v\cdot w\, \mathrm{d}\mu$
\\
\hline
Fisher--Rao (FR)
&
$\displaystyle T_\mu = \{\varphi:\int \varphi\, \mathrm{d}\mu=0\}$
&
$\displaystyle g_\mu(\varphi,\psi)=\int_\Omega \varphi\psi\, \mathrm{d}\mu$
\\
\hline
WFR (HK)
&
$\displaystyle T_\mu^{\mathrm{WFR}}
= \overline{\{(\phi,\nabla\phi)\}}^{L^2(\mu)}$
&
$\displaystyle g_\mu^{\mathrm{WFR}}\big((\phi,\nabla\phi),(\psi,\nabla\psi)\big)
=\int_\Omega(\phi\psi+\nabla\phi\cdot\nabla\psi)\,\mathrm{d}\mu$
\\
\hline
\end{tabular}
\caption{Tangent spaces and Riemannian metrics.}
\end{table*}

In the probabilistic (balanced) setting we again subtract the mean in the first component to enforce $\int\phi\,\mathrm{d}\mu=0$, i.e.
\[
\nabla^{\mathrm{WFR}} F(\mu)
=
\left(
\frac{\delta F}{\delta\mu}(\mu) - \int_\Omega \frac{\delta F}{\delta\mu}(\mu)\,\mathrm{d}\mu,\,
\nabla\!\left(\frac{\delta F}{\delta\mu}(\mu)\right)
\right).
\]

The WFR gradient flow is defined by following the steepest descent direction:
\[
(\varphi_t,v_t) = - \nabla^{\mathrm{WFR}} F(\mu_t).
\]
Plugging into \eqref{eq:wfr-continuity} gives the PDE
\begin{equation}
\partial_t \mu_t
+ \nabla\!\cdot\!\left(
\mu_t\,\nabla \!\left( \frac{\delta F}{\delta\mu}(\mu_t) \right)
\right)
=
- \left(
\frac{\delta F}{\delta\mu}(\mu_t)
- \int_\Omega \frac{\delta F}{\delta\mu}(\mu_t)\,\mathrm{d}\mu_t
\right)\mu_t.
\label{eq:wfr-flow}
\end{equation}
In the unbalanced case one drops the mean-subtraction term. The left-hand side corresponds to Wasserstein-type transport, while the right-hand side corresponds to Fisher--Rao-type reaction. As in the pure Wasserstein and Fisher--Rao cases,
\[
\frac{\mathrm{d}}{\mathrm{d}t}F(\mu_t)
= - \big\|\nabla^{\mathrm{WFR}} F(\mu_t)\big\|_{\mu_t}^2 \leq 0,
\]
so the flow is again a steepest descent of $F$, now in the WFR geometry.

\begin{table*}[h!]
\centering
\renewcommand{\arraystretch}{1.9}
\begin{tabular}{|c|c|}
\hline
\textbf{Geometry} & \textbf{Geodesics} \\ \hline
Wasserstein ($W_2$)
&
$\displaystyle 
\partial_t \mu_t + \nabla\!\cdot(\mu_t \nabla\phi_t)=0, \quad
\partial_t \phi_t + \tfrac12 |\nabla\phi_t|^2 = c(t).$
\\[2pt]
&
(McCann displacement interpolation: 
$\mu_t = ((1-t)\mathrm{id} + t T_{\mu_0\to\mu_1})_\#\mu_0$)
\\
\hline
Fisher--Rao (FR)
&
$\displaystyle \sqrt{\mu_t} = (1-t)\sqrt{\mu_0}+t\sqrt{\mu_1};\qquad
\mu_t = \big((1-t)\sqrt{\mu_0}+t\sqrt{\mu_1}\big)^2.$
\\
&
(Geodesics are great circles in $L^2$ restricted to positive cone.)
\\
\hline
WFR (HK)
&
Minimizers of the action:
$\displaystyle 
\int_0^1\!\!\int_\Omega (\phi_t^2 + |\nabla\phi_t|^2)\,\mathrm{d}\mu_t \mathrm{d}t$
subject to
$\partial_t\mu_t+\nabla\cdot(\mu_t\nabla\phi_t)=\phi_t\mu_t.$
\\
\hline
\end{tabular}
\caption{Geodesics in Wasserstein, Fisher--Rao, and WFR geometries.}
\end{table*}

We conclude this section by summarizing the three geometric frameworks introduced above—
Wasserstein, Fisher--Rao, and Wasserstein--Fisher--Rao (WFR)—highlighting their structural
similarities and differences through their tangent spaces, metrics, geodesics, curvature
properties, and associated gradient flows.  Taken together, these geometries form a
coherent hierarchy: Fisher--Rao encodes pure reaction, Wasserstein encodes pure transport,
and WFR merges the two in a Sobolev-type structure.

\begin{table*}[h!]
\centering
\renewcommand{\arraystretch}{1.9}
\begin{tabular}{|c|c|}
\hline
\textbf{Geometry} & \textbf{Gradient Flow of $F$} \\ \hline
Wasserstein ($W_2$)
&
$\displaystyle
\partial_t \mu_t = 
\nabla\!\cdot\!\big(\mu_t \nabla(\tfrac{\delta F}{\delta\mu}(\mu_t))\big)
$
\\[1pt]
&
(Steepest descent in $W_2$, e.g.\ heat flow for $F(\mu)=\int \mu\log\mu$)
\\
\hline
Fisher--Rao (FR)
&
$\displaystyle 
\partial_t \mu_t = 
- \Big(\tfrac{\delta F}{\delta\mu} 
 - \!\int\tfrac{\delta F}{\delta\mu}\,\mathrm{d}\mu \Big)\mu_t
$
\\
&
(Pointwise exponential reweighting.)
\\
\hline
WFR (HK)
&
$\displaystyle
\partial_t\mu_t + \nabla\!\cdot(\mu_t\nabla \tfrac{\delta F}{\delta\mu})
= -\Big(\tfrac{\delta F}{\delta\mu}
     -\!\int\tfrac{\delta F}{\delta\mu}\,\mathrm{d}\mu\Big)\mu_t
$
\\
&
(Combined transport + reaction.)
\\
\hline
\end{tabular}
\caption{Gradient flows as steepest descents in each geometry.}
\end{table*}

A core theme in the geometric analysis of probability spaces is the intimate relation between
\emph{curvature}, \emph{convexity}, and \emph{stability} of associated differential equations.

\begin{itemize}
\item In \textbf{Wasserstein geometry}, Ricci-like curvature enters through
the Bakry--Émery $\Gamma_2$ condition.  
If $F$ is $\lambda$-geodesically convex in $W_2$ (i.e.\ $F''(\mu_t)\ge\lambda$ along $W_2$ geodesics),
then the $W_2$-gradient flow of $F$ is contractive:
\[
W_2(\mu_t,\nu_t) \le e^{-\lambda t} W_2(\mu_0,\nu_0),
\]
and this corresponds to exponential decay of solutions of Fokker--Planck-type PDEs.

\item In \textbf{Fisher--Rao geometry}, the metric is flat when lifted to the Hellinger cone, 
but submanifolds (e.g.\ exponential families) often exhibit \emph{negative sectional curvature},
reflecting the strict convexity of the log-partition function.  
This curvature controls the stability of statistical estimators and likelihood flows.

\item In \textbf{WFR geometry}, curvature arises from a combination of Wasserstein's 
(second-order, transport-based) geometry and Fisher--Rao's (first-order, reaction-based)
geometry.  
The WFR metric induces a generalized convexity notion for functionals:
geodesic convexity in WFR implies
exponential convergence of the mixed transport--reaction PDE \eqref{eq:wfr-flow}.  
\end{itemize}

In all three geometries, convexity of an energy functional $F$ along the appropriate geodesics
determines the long-time behavior of its gradient flow, providing a unified explanation of
why Wasserstein, Fisher--Rao, and WFR evolutions exhibit exponential convergence,
entropy dissipation, or contractive stability.

Each geometry supports a coherent differential calculus (tangent spaces, gradients,
connections, curvature), a natural geodesic structure, and a steepest-descent interpretation
of PDEs.  
The hybrid viewpoint is essential for applications in which both spatial rearrangement and 
local mass variation play critical roles—ranging from diffusion--reaction systems to
generative models and geometric flows.

Now, before going to further analyze these concepts, using the theory of linear operators, we try to come up with an sde for each of the pdes introduced above.

\section{From PDE to SDE}
\label{sec:pde to sde}

The previous section endowed spaces of measures with Riemannian structures of
Wasserstein, Fisher--Rao, and Wasserstein--Fisher--Rao type, that potentially can help us to sample better, and faster.  In the present section we take the complementary
viewpoint: starting from a partial differential equation (PDE) for
probability densities, we construct weighted stochastic processes on the underlying
state space whose laws realize the prescribed evolution.  In other words, we
systematically pass \emph{from PDE to SDE}, enabling implementation.

Now that we have seen the different pdes provided in the previous section, we provide a corresponding SDE, in this case a weighted sde, to be able to simulate the aforementioned pdes. In general consider:
\begin{multline}
\frac{\partial}{\partial t} p_t(x_t)
= -\nabla \cdot \big( p_t(x_t)\, v_t(x_t) \big)
+ \frac{\sigma_t^2}{2}\, \Delta p_t(x_t)
\\+ p_t(x_t)\!\left( \psi_t(x_t) - \int \psi_t(x_t)\, p_t(x_t)\, dx_t \right)
\label{eq:pde to wsde}\tag{General PDE}
\end{multline}

where to sample from $p_t(x)$, one first has to sample $x_t$ via
The following SDE
\begin{equation}
\mathrm{d}x_t = v_t(x_t)\mathrm{d}t + \sigma_t \mathrm{d}B_t,
\qquad
\mathrm{d}w_t = \bar \psi_t(x_t)\mathrm{d}t,
\end{equation}
and then reweight the obtained samples using $w_t$. 

In practice, we can account for this difference by sampling
\begin{equation}
i \sim \mathrm{Categorical}\left\{
\frac{\exp(w_T^k)}{\sum_{j=1}^K \exp(w_T^j)}
\right\}_{k=1}^K,
\end{equation}
and returning $x_T^{(i)}$ as an approximate sample from $p_T$. For estimating the expectation of test
functions $\phi$, we account for the weights by reweighting a collection of $K$
particles, i.e.,
\begin{equation}
\mathbb{E}_{p_T}[\phi(x)]
\approx
\sum_{k=1}^K
\frac{\exp(w_T^k)}{\sum_j \exp(w_T^j)}\, \phi(x_T^k).\label{eq:sampling}
\end{equation}

In other words, one can start $K$ particles $X^i_t$ each following a drift and diffusion and proceed the spatio movement in a discrete version until interval $[0, t_1]$, Along the path, we calculate $w^i_{t_1} = \exp(\psi_t(X^i_{t_1}))$, and then, resample the particles with these weights. As proved by the theorem below, this would indeed realizes the pde.

\textbf{Remark}. The exploration in space, is still done by the Brownian motion. Though, by introducing this kill/birth process done by weights, we are relocating particles to mimic the teleportation we need. There is a Jump Process interpertation of Weighted SDE above as well discussed later in the text.

\begin{theorem}
    The Weighted SDE, and sampling scheme introduced above, would realize the generalized pde,i.e.
    $$\text{law}(X_t^i)\sim p^t $$
\end{theorem}
\textbf{Proof.} We proceed in two steps, first finding a Kolmogorov backward equation corresponding to evolution under a weighted
Feynman–Kac SDE. We then use this identity to derive the expectation estimator.

\textbf{Proposition A.1.}
For a bounded test function $\phi : X \rightarrow \mathbb{R}$ and $p_t$ satisfying Eq.~\eqref{eq:pde to wsde}, we have
\begin{equation}
\mathbb{E}_{p_T(x_T)}[\phi(x_T)]
= \frac{1}{Z_T}\;
\mathbb{E}\!\left[
\exp\!\left( \int_0^T \psi_s(x_s)\, ds \right)\, \phi(x_T)
\right]
\label{eq:fk-estimator}
\end{equation}
where $dx_t = v_t(x_t)\, dt + \sigma_t\, dB_t$, $x_0 \sim p_0$.

Here $Z_T$ is a normalization constant independent of $x$.  
Eq.~\eqref{eq:fk-estimator} suggests that the self-normalized importance sampling approximation in Eq. \ref{eq:sampling}
is consistent as $K \to \infty$.

\textbf{Proof.}
The proof proceeds in three steps. We first derive the backward Kolmogorov
equation for appropriate functions, then specify the evolution of the \ref{eq:pde to wsde} for the unnormalized density

For a given test function $\phi(x)$, consider defining the following function:
\begin{multline}
\Phi_T(x,t)
= \mathbb{E}\!\left[
\exp\!\left( \int_t^T \psi_s(x_s)\, ds \right)\,
\phi(x_T)
\;\middle|\; x_t = x
\right],
\\ \Phi_T(x,T) = \phi(x).
\label{eq:Phi-def}
\end{multline}

For any $\tau > t$, we have
\begin{multline}
\Phi_T(x,t)
\\=
\mathbb{E}\!\left[
\exp\!\left(\int_t^\tau \psi_s(x_s)\, ds\right)
\exp\!\left(\int_\tau^T \psi_s(x_s)\, ds\right)\,
\phi(x_T)
\;\middle|\; x_t = x
\right]
\\=
\mathbb{E}\!\left[
\exp\!\left(\int_t^\tau \psi_s(x_s)\, ds\right)\,
\Phi_T(x_\tau,\tau)
\;\middle|\; x_t = x
\right]
\label{eq:Phi-identity}
\end{multline}

To relate $\Phi_T(x,t)$ and the expected value of $\Phi_T(x_\tau,\tau)$, we apply Itô's product rule. For $\tau \mapsto x_\tau$:
\begin{multline}
d\!\Bigl(
e^{\int_t^\tau \psi_s ds}\,
\Phi_T(x_\tau,\tau)
\Bigr)
=
e^{\int_t^\tau \psi_s ds}
\Bigl(
\frac{\partial \Phi_T}{\partial \tau}
+ \langle v_\tau(x_\tau), \nabla \Phi_T(x_\tau,\tau) \rangle
+ \frac{\sigma_\tau^2}{2}\, \Delta \Phi_T(x_\tau,\tau)
+ \psi_\tau(x_\tau)\, \Phi_T(x_\tau,\tau)
\Bigr)\, d\tau
\\
+
e^{\int_t^\tau \psi_s ds}
\, \sigma_\tau\, \langle \nabla\Phi_T, dW_\tau \rangle.
\label{eq:Ito-product}
\end{multline}

Taking expectations removes the martingale term.  
Thus $\Phi_T$ satisfies the backward PDE:
\begin{multline}
\frac{\partial \Phi_T(x_\tau,\tau)}{\partial \tau}
+ \langle v_\tau(x_\tau), \nabla \Phi_T(x_\tau,\tau) \rangle
\\+ \frac{\sigma_\tau^2}{2}\, \Delta \Phi_T(x_\tau,\tau)
+ \psi_\tau(x_\tau)\, \Phi_T(x_\tau,\tau)
= 0.
\label{eq:backward-pde}
\end{multline}

In practice, we cannot exactly calculate
\[
\int \psi_t(x_t)\, p_t(x_t)\, dx_t,
\]
which appears in  Eq.\ (\ref{eq:p-normalized-evolution}) below to ensure normalization.

For now, consider the evolution of the \emph{unnormalized} density
\[
\tilde p_t(x) = p_t(x)\, Z_t,
\]
for a particular $v_t, \sigma_t, \psi_t$ and normalization constant $Z_t$.  
With foresight, we define:
\begin{equation}
\frac{\partial}{\partial t} \tilde p_t(x_t)
=
-\nabla \cdot \big( \tilde p_t(x_t)\, v_t(x_t) \big)
+ \frac{\sigma_t^2}{2}\, \Delta \tilde p_t(x_t)
+ \tilde p_t(x_t)\, \psi_t(x_t).
\label{eq:ptilde-evolution}
\end{equation}

We further define the normalization evolution:
\begin{equation}
\partial_t \log Z_t := 
\int p_t(x)\, \psi_t(x)\, dx.
\label{eq:Z-evolution}
\end{equation}

This choice is motivated by the reweighting-only evolution
\[
\partial_t p^w_t(x)
=
p^w_t(x)\big( \psi_t(x) - \int p^w_t(x)\, \psi_t(x)\, dx \big),
\]
which implies:
\[
\partial_t \log p^w_t(x)
=
\psi_t(x) - \int p^w_t(x)\, \psi_t(x)\, dx.
\]

Writing $p^w_t(x) = \tilde p^w_t(x) / Z_t$ and equating derivatives yields Eq.\ \eqref{eq:Z-evolution}. We now verify that the definitions \eqref{eq:ptilde-evolution}–\eqref{eq:Z-evolution} are consistent with the original pde:
\begin{multline}
\frac{\partial}{\partial t} p_t(x_t)
=
-\nabla \cdot \big( p_t(x_t)\, v_t(x_t) \big)
+ \frac{\sigma_t^2}{2}\, \Delta p_t(x_t)
\\+ p_t(x_t)\!\left(
\psi_t(x_t) - \int \psi_t(x_t)\, p_t(x_t)\, dx_t
\right).
\label{eq:p-normalized-evolution}
\end{multline}

Since $p_t(x) = \tilde p_t(x)\, Z_t^{-1}$, we compute:
\begin{equation}
\frac{\partial}{\partial t} p_t(x_t)
=
\frac{\partial}{\partial t}
\left( \tilde p_t(x_t)\, Z_t^{-1} \right)
=
Z_t^{-1} \frac{\partial \tilde p_t}{\partial t}
\;+\;
\tilde p_t(x_t)\, Z_t^{-1}\, \partial_t(Z_t^{-1}).
\label{eq:pt-derivative-step1}
\end{equation}

Using
\[
\partial_t(Z_t^{-1}) = -Z_t^{-1} \partial_t \log Z_t,
\]
Eq.\ \eqref{eq:pt-derivative-step1} becomes:
\begin{equation}
\frac{\partial}{\partial t} p_t(x_t)
=
Z_t^{-1}
\frac{\partial \tilde p_t}{\partial t}
-
\tilde p_t(x_t)\, Z_t^{-1}\, \partial_t \log Z_t.
\label{eq:pt-derivative-step2}
\end{equation}

Inserting Eq.\ \eqref{eq:ptilde-evolution}:
\begin{multline}
\frac{\partial}{\partial t} p_t(x_t)=Z_t^{-1}\left(-\nabla \cdot ( \tilde p_t v_t )+ \frac{\sigma_t^2}{2}\, \Delta\tilde p_t+ \tilde p_t\, \psi_t\right)\\-\tilde p_t Z_t^{-1}\int p_t(x)\, \psi_t(x)\, dx.\label{eq:pt-derivative-step3}
\end{multline}

Since $\nabla Z_t = 0$, we may move $Z_t^{-1}$ inside derivatives:
\begin{multline}
\frac{\partial}{\partial t} p_t(x_t)
=
-\nabla \cdot ( p_t(x_t)\, v_t(x_t) )
+ \frac{\sigma_t^2}{2}\, \Delta p_t(x_t)
\\+ p_t(x_t)\, \psi_t(x_t)
- p_t(x_t) \int p_t(x)\, \psi_t(x)\, dx.
\label{eq:pt-derivative-final}
\end{multline}

This matches precisely the original PDE \eqref{eq:p-normalized-evolution}, completing the consistency check.

\textbf{Expectation Estimation.}
Now, we use Eq.~\eqref{eq:backward-pde} to write the total derivative of the following integral under the unnormalized
density $\tilde p_t(x)$:
\begin{multline}
\frac{d}{dt} \left[ \int \Phi_T(x,t)\, \tilde p_t(x)\, dx \right]
=
\int \left( \frac{\partial \Phi_T(x,t)}{\partial t} \right)
\tilde p_t(x)\, dx
\;\\+\;
\int \Phi_T(x,t)
\left( \frac{\partial \tilde p_t(x)}{\partial t} \right) dx.
\label{eq:total-derivative}
\end{multline}

Using Eqs.~\eqref{eq:backward-pde} and \eqref{eq:ptilde-evolution}, we obtain
\begin{multline}
\frac{d}{dt} \int \Phi_T(x,t)\, \tilde p_t(x)\, dx
=
\\\int \Big(
- \langle v_t(x), \nabla \Phi_T(x,t) \rangle
- \tfrac{\sigma_t^2}{2}\, \Delta \Phi_T(x,t)
- \Phi_T(x,t)\, \psi_t(x)
\Big)\, \tilde p_t(x)\, dx
\\
\quad
+
\int \Phi_T(x,t)
\Big(
- \nabla \cdot (\tilde p_t(x)\, v_t(x))
+ \tfrac{\sigma_t^2}{2}\, \Delta \tilde p_t(x)
+ \tilde p_t(x)\, \psi_t(x)
\Big)\, dx. 
\label{eq:derivative-expanded-1}
\end{multline}

Integrating by parts, the second line becomes:
\begin{multline}
\frac{d}{dt} \int \Phi_T(x,t)\, \tilde p_t(x)\, dx
\\=
\int
\Big(
- \langle v_t(x), \nabla \Phi_T(x,t) \rangle
- \tfrac{\sigma_t^2}{2}\, \Delta \Phi_T(x,t)
- \Phi_T(x,t)\, \psi_t(x)
\Big)\, \tilde dp_t(x)
\label{eq:derivative-expanded-2}
\\
\quad+
\int
\Big(
\langle v_t(x), \nabla \Phi_T(x,t) \rangle
+ \tfrac{\sigma_t^2}{2}\, \Delta \Phi_T(x,t)
+ \Phi_T(x,t)\, \psi_t(x)
\Big)\, \tilde dp_t(x)
\nonumber
\\= 0.
\nonumber
\end{multline}

Thus,
\begin{equation}
\frac{d}{dt} \int \Phi_T(x,t)\, \tilde p_t(x)\, dx = 0.
\label{eq:derivative-zero}
\end{equation}

Integrating over $t \in [0,T]$, we obtain:
\begin{multline}
\int \Phi_T(x_T,T)\, \tilde p_T(x_T)\, dx_T
-
\int \Phi_T(x_0,0)\, \tilde p_0(x_0)\, dx_0
=
\\ \int_0^T \frac{d}{dt}
\left[
\int \Phi_T(x,t)\, \tilde p_t(x)\, dx
\right] dt
= 0.
\label{eq:Phi-integral-constancy}
\end{multline}

Hence these two quantities must be equal:
\begin{equation}
\int \Phi_T(x_0,0)\, \tilde p_0(x_0)\, dx_0
=
\int \Phi_T(x_T,T)\, \tilde p_T(x_T)\, dx_T.
\label{eq:Phi-equality}
\end{equation}

Using $\tilde p_0 = p_0$ and $Z_0 = 1$, and the fact that  
\[
\Phi_T(x_T,T) = \phi(x_T),
\]
we have:
\begin{equation}
\int
\mathbb{E}\!\left[
e^{\int_0^T \psi_s(x_s)\, ds}\, \phi(x_T)
\;\middle|\; x_0
\right] p_0(x_0)\, dx_0
=
Z_T \int \phi(x_T)\, p_T(x_T)\, dx_T.
\label{eq:expectation-balance}
\end{equation}

Thus,
\begin{equation}
\frac{1}{Z_T}
\mathbb{E}\!\left[
e^{\int_0^T \psi_s(x_s)\, ds}\, \phi(x_T)
\right]
=
\mathbb{E}_{p_T}[\phi(x_T)],
\label{eq:fk-final-identity}
\end{equation}
which is precisely the identity claimed in Prop.~A.1.

\medskip

In practice, we approximate
\begin{equation}
Z_T \approx \frac{1}{K} \sum_{k=1}^K 
\exp\!\left( \int_0^T \psi_s(x_s^{(k)})\, ds \right)
=
\frac{1}{K} \sum_{k=1}^K e^{w_T^{(k)}},
\label{eq:ZT-estimator}
\end{equation}
and similarly
\begin{equation}
\mathbb{E}
\left[
e^{\int_0^T \psi_s(x_s)\, ds}\, \phi(x_T)
\right]
\approx
\frac{1}{K}
\sum_{k=1}^K
e^{w_T^{(k)}}\, \phi\!\left( x_T^{(k)} \right).
\label{eq:weighted-expectation-approx}
\end{equation}

This yields Eq.\ref{eq:sampling}  
We emphasize that the choice of terminal time $T$ was arbitrary; the same reasoning applies to any intermediate $t$, which
implies that samples are correctly weighted for estimating expectations at all intermediate times.

   \textbf{Jump processes and pure reaction dynamics.}
        A Markov jump process is determined by a rate function
$\lambda_t(x)$, which governs the frequency of jump events, and
a Markov transition kernel $J_t(y|x)$, which is used to sample
the next state when a jump occurs. The forward Kolmogorov
equation for a jump process is given by
\begin{multline}
\frac{\partial p_t^{\mathrm{jump}}(x)}{\partial t}
=
\left(
\int \lambda_t(y) J_t(x|y) p_t(y)\,\mathrm{d}y
\right)
-
p_t(x)\lambda_t(x)
\end{multline}
where the two terms can intuitively be seen to measure the
inflow and outflow of probability due to jumps.

One could find $\lambda_t(x), J_t(y|x)$ such that
$p_t^{\mathrm{jump}}$ matches the evolution of $p_t^{w}$ in \ref{eq:pde to wsde} for a given choice of $\psi_t$. For a given $\psi_t$, define the jump process rate and
transition as
\begin{align}
\lambda_t(x)
&=
\bigl(\psi_t(x)-\mathbb{E}_{p_t}[\psi_t]\bigr)^-,
\\
J_t(y|x)
&=
\frac{\bigl(\psi_t(y)-\mathbb{E}_{p_t}[\psi_t]\bigr)^+\,p_t(y)}
{\int \bigl(\psi_t(z)-\mathbb{E}_{p_t}[\psi_t]\bigr)^+\,p_t(z)\,\mathrm{d}z},
\end{align}
where $(u)^- := \max(0,-u)$ and $(u)^+ := \max(0,u)$.
Then,
\begin{equation}
\frac{\partial p_t^{\mathrm{jump}}(x)}{\partial t}
=
\frac{\partial p_t^{w}(x)}{\partial t}
=
p_t(x)\bigl(\psi_t(x)-\mathbb{E}_{p_t}[\psi_t]\bigr).
\end{equation}

In continuous time and the mean-field limit, this jump process
formulation of reweighting corresponds to simulating
\begin{equation}
x_{t+\mathrm{d}t}
=
\begin{cases}
x_t
& \text{w.p. } 1-\lambda_t(x_t)\,\mathrm{d}t + o(\mathrm{d}t),\\
\sim J_t(y|x_t)
& \text{w.p. } \lambda_t(x_t)\,\mathrm{d}t + o(\mathrm{d}t).
\end{cases}
\tag{33}
\end{equation}

We expect this process to improve the sample population in
efficient fashion, since jump events are triggered only in
states where $(\psi_t(x)-\mathbb{E}_{p_t}[\psi_t])^- \ge 0 \iff
\psi_t(x) \le \mathbb{E}_{p_t}[\psi_t]$, and transitions are more
likely to jump to states with high excess weight
$(\psi_t(y)-\mathbb{E}_{p_t}[\psi_t])^+ > 0$.

\textbf{Proof}\cite{DelMoral2013MeanField}. One way to perform a simulation of the reweighting equation is to rewrite it as a jump process.
We recall the definition of the Markov generator of a jump process.
Let $W_t(x,y) = \lambda_t(x)\, J_t(y|x)$ where $J_t(y|x)$ is normalized.
Then
\begin{equation}
J_t^{(W)}[\phi](x)
=
\int \big( \phi(y) - \phi(x) \big)\,
\lambda_t(x)\, J_t(y|x)\, dy,
\label{eq:jump-gen}
\end{equation}
and the adjoint generator satisfies
\begin{equation}
J_t^{*(W)}[p_t](x)
=
\left(
\int \lambda_t(y)\, J_t(x|y)\, p_t(y)\, dy
\right)
- p_t(x)\, \lambda_t(x).
\label{eq:jump-adjoint}
\end{equation}

\textbf{Proof.}
\begin{multline}
\int \phi(x)\, J_t^{*}[p_t](x)\, dx
=
\int J_t[\phi](x)\, p_t(x)\, dx
\label{eq:jump-proof-1}
\\
=
\int \!\int 
\big( \phi(y) - \phi(x) \big)
\lambda_t(x) J_t(y|x)\, dy\, p_t(x)\, dx
\nonumber
\\
=
\int\!\int \phi(y)\, \lambda_t(x) J_t(y|x) p_t(x)\, dy\, dx
\\-
\int\!\int \phi(x)\, \lambda_t(x) J_t(y|x) p_t(x)\, dy\, dx
\nonumber
\\
=
\int\!\int \phi(x)\, \lambda_t(y) J_t(x|y) p_t(y)\, dx\, dy
\\-
\int \phi(x)\, \lambda_t(x) p_t(x)
\left( \int J_t(y|x)\, dy \right) dx
\nonumber
\\
=
\int \phi(x)
\left[
\int \lambda_t(y) J_t(x|y) p_t(y)\, dy
- p_t(x) \lambda_t(x)
\right] dx,
\nonumber
\end{multline}
proving Eq.\ \eqref{eq:jump-adjoint}.

We aim to construct a jump process whose adjoint generator matches the reweighting generator:
\begin{equation}
J_t^{*(W)}[p_t](x) = L_t^{*(g)}[p_t](x).
\label{eq:jump-equals-reweight}
\end{equation}

Define the positive/negative parts:
\[
(u)_- := \max(0, -u),
\qquad
(u)_+ := \max(0, u),
\qquad
(u)_+ - (u)_- = u.
\]

Following Angeli et al.\ (2019) and Del Moral (2013), we define:
\begin{multline}
\lambda_t(x)
=
\big( \psi_t(x) - \mathbb{E}_{p_t}[\psi_t] \big)_-,
\\
J_t(y|x)
=
\frac{
\big( \psi_t(y) - \mathbb{E}_{p_t}[\psi_t] \big)_+\, p_t(y)
}{
\int \big( \psi_t(z) - \mathbb{E}_{p_t}[\psi_t] \big)_+\, p_t(z)\, dz
}.
\label{eq:lambda-J-def}
\end{multline}

For the choice of $\lambda_t$ and $J_t$ in \eqref{eq:lambda-J-def}, the adjoint generator satisfies
\begin{equation}
J_t^{*(W)}[p_t](x)
=
L_t^{*(g)}[p_t](x)
=
p_t(x)\!\left(
\psi_t(x) - \int \psi_t(x)\, p_t(x)\, dx
\right).
\label{eq:prop-B3}
\end{equation}

\textbf{Proof.}
Expanding Eq.\ \eqref{eq:jump-adjoint}:
\begin{multline}
J_t^{*(W)}[p_t](x)
=
\int
\big( \psi_t(y) - \mathbb{E}_{p_t}[\psi_t] \big)_-
\frac{
\big( \psi_t(x) - \mathbb{E}_{p_t}[\psi_t] \big)_+\, p_t(y)
}{
\int (\psi_t(z) - \mathbb{E}_{p_t}[\psi_t])_+\, p_t(z)\, dz
} dy
\\-
p_t(x)\, \big( \psi_t(x) - \mathbb{E}_{p_t}[\psi_t] \big)_-.
\label{eq:jump-substitution}
\end{multline}

Using the identity
\[
\int (\psi_t(z)-\mu)_+ p_t(z)\, dz
=
\int (\psi_t(z)-\mu)_- p_t(z)\, dz
\]
(which follows from $\mathbb{E}_{p_t}[\psi_t-\mu]=0$), and splitting cases $\psi_t(x)\ge \mu$ or not, one obtains:
\[
J_t^{*(W)}[p_t](x)
=
p_t(x) \big( \psi_t(x) - \mathbb{E}_{p_t}[\psi_t] \big)
=
L_t^{*(g)}[p_t](x).
\]

\begin{table*}[t]
\centering
\renewcommand{\arraystretch}{1.5}
\begin{tabular}{|c|c|c|}
\hline
\textbf{Component}
& \textbf{PDE Contribution}
& \textbf{Generator $L_t$ acting on $\phi$}\\
\hline\hline

\textbf{Drift}
&
$\displaystyle
\partial_t p_t
=
-\nabla\!\cdot\!\bigl(p_t v_t\bigr)
$
&
$\displaystyle
L_t^{(v)}[\phi](x)
=
\langle \nabla \phi(x), v_t(x) \rangle
$
\\
\hline

\textbf{Diffusion}
&
$\displaystyle
\partial_t p_t
=
\frac{\sigma_t^2}{2}\Delta p_t
$
&
$\displaystyle
L_t^{(\sigma)}[\phi](x)
=
\frac{\sigma_t^2}{2}\Delta \phi(x)
$
\\
\hline

\textbf{Reweighting}
&
$\displaystyle
\partial_t p_t
=
p_t(x)\!\left(\psi_t(x)-\int \psi_t\,\mathrm d p_t\right)
$
&
$\displaystyle
L_t^{(g,p)}[\phi](x)
=
\phi(x)\!\left(\psi_t(x)-\int \psi_t(x)\,p_t(x)\,\mathrm dx\right)
$
\\
\hline

\end{tabular}

\caption{
Decomposition of the pde into drift, diffusion, and
reweighting components, showing the corresponding PDE contribution,
infinitesimal generator acting on test functions.
}
\label{tab:generator-decomposition}
\end{table*}

\section{Conclusion and Future Directions}

This work presents a foundational investigation into the role of information-geometric structure in
diffusion-based sampling, with particular emphasis on Wasserstein--Fisher--Rao (WFR) geometry as a
unifying framework for transport and reweighting mechanisms.
By interpreting diffusion, weighted diffusion, and reaction dynamics through a geometric and
operator-theoretic lens, we demonstrate how classical Ornstein--Uhlenbeck–type sampling processes
can be systematically augmented without abandoning their underlying stochastic calculus.
\begin{figure*}
    \centering
    \includegraphics[width=0.8\linewidth]{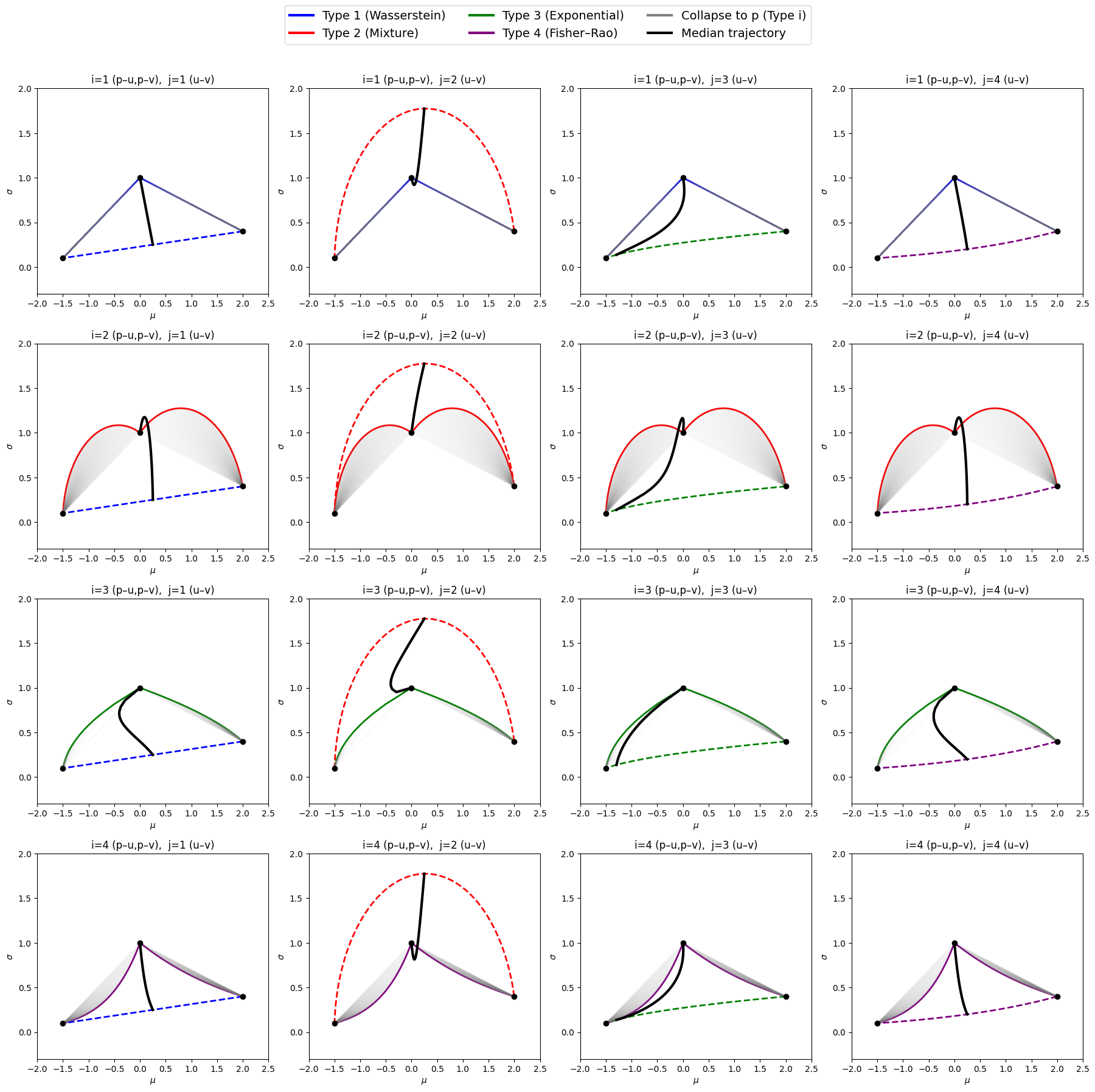}
    \caption{Geodesic structure and median trajectories across different geometries in the $(\mu,\sigma)$ parameter space. Each panel depicts a triangle formed by three distributions $(p,u,v)$, where the edges connecting $p$ to $u$ and $p$ to $v$ are constructed using a fixed geometry indexed by $i\in\{1,2,3,4\}$, while the edge connecting $u$ and $v$ is generated using a (possibly different) geometry indexed by $j\in\{1,2,3,4\}$. Type~1 corresponds to Wasserstein geodesics (blue), Type~2 to linear mixture geodesics (red), Type~3 to exponential geodesics (green), and Type~4 to Fisher--Rao geodesics (purple). Gray segments indicate collapse trajectories toward the reference distribution $p$ under the geometry indexed by $i$, and the black curve denotes the induced median trajectory between $u$ and $v$, obtained by projecting the $u$--$v$ geodesic through $p$ according to the pair $(i,j)$. Dashed curves represent the direct $u-v$ geodesic in geometry $j$, while solid colored curves illustrate the lifted or corrected paths resulting from the interaction between the two geometric structures. The collection of panels highlights how mismatches between transport and information-geometric structures modify both geodesic shapes and the resulting median trajectories.}
    \label{fig:figure}
\end{figure*}
Beyond its immediate constructions, this study opens several mathematically substantive directions
for future research, including(but not limited to) the following:

\begin{itemize}
    \item \textbf{Spectral and semigroup analysis of corrected generators.}
    A comprehensive understanding of diffusion--reweighting and diffusion--jump dynamics requires
    the spectral analysis of the associated (typically unbounded and non-selfadjoint) generators.
    Tools from semigroup theory, most notably the Hille--Yosida theorem and its extensions to
    non-conservative and nonlocal operators, provide a natural framework for characterizing
    well-posedness, ergodicity, and convergence rates of these corrected processes.
    Establishing quantitative links between geometric corrections and spectral gaps remains a
    central open problem.

    \item \textbf{Interaction of geometric geodesics and sampling efficiency.}
    In classical diffusion models, entropic optimal transport (equivalently, Schrödinger bridges)
    characterizes the geodesic structure induced by Ornstein--Uhlenbeck dynamics.
    Within the WFR framework, however, additional geodesic families arise—such as mixture (linear),
    exponential, and information-preserving interpolations that each encodes a distinct tradeoff between
    transport and mass variation.
    Understanding the interaction and possible hybridization of these geodesics is essential for
    designing sampling schemes capable of traversing mixture distributions and multimodal targets
    \emph{without} retraining score functions, thereby enabling geometry-driven adaptation at the
    level of inference rather than model fitting(Look Fig. \ref{fig:figure} and Last Appendix).
\end{itemize}

We believe that further development along these directions will deepen the theoretical foundations
of diffusion-based generative modeling and contribute to a principled synthesis of stochastic
analysis, optimal transport, and information geometry in modern sampling theory.
\bibliographystyle{plain}
\bibliography{refrences}

\newpage
\appendix
\section{More On Wasserstein Geometry}

Once $(T_\mu\mathcal{P}_2,\langle\cdot,\cdot\rangle_\mu)$ is identified with the closure of gradients,
one may introduce a formal Levi--Civita connection $\nabla^{W_2}$ on vector fields along curves.
For smooth potentials $\phi_t$, $\psi_t$, with corresponding velocity fields
$v_t=\nabla \phi_t$, $w_t=\nabla\psi_t$, the covariant derivative along $\mu_t$ is
\[
\frac{D}{dt} w_t
:= \Pi_{\mu_t}\bigl( \partial_t w_t + (v_t\cdot\nabla) w_t \bigr),
\]
where $\Pi_{\mu_t}$ denotes projection onto $T_{\mu_t}\mathcal{P}_2$, i.e.\ projection onto gradient vector fields in $L^2(\mu_t)$.
Explicitly, any vector field can be Helmholtz-decomposed as
\[
Z = \nabla \zeta + Z^\perp,\qquad Z^\perp\ \text{$\mu$-divergence-free},
\]
and $\Pi_\mu Z = \nabla\zeta$.
Thus the covariant derivative simplifies to
\[
\frac{D}{dt} w_t = \nabla \pi_t,
\quad\text{where}\quad
\Delta_{\mu_t} \pi_t = \nabla\!\cdot(\mu_t ( \partial_t w_t + (v_t\cdot\nabla)w_t)),
\]
with $\Delta_{\mu} := \nabla\!\cdot(\mu \nabla)$.  

This connection is metric-compatible and torsion-free (formally), hence the unique Levi--Civita connection of the $W_2$ Riemannian structure.

A curve $(\mu_t,v_t=\nabla\phi_t)$ is a geodesic iff its velocity field is parallel transported:
\[
\frac{D}{dt} v_t=0.
\]
Using the covariant derivative above, this yields
\[
\partial_t \nabla\phi_t + (\nabla\phi_t\cdot\nabla)\nabla\phi_t 
= \nabla \pi_t
\]
for some scalar pressure $\pi_t$ that enforces the gradient constraint.
Taking divergence w.r.t.\ $\mu_t$ shows that $\pi_t$ ensures $\nabla\phi_t$ remains in the tangent space.
In Lagrangian coordinates where $X_t$ satisfies $\dot X_t = \nabla\phi_t(X_t)$, the geodesic equation becomes the classical Hamilton–Jacobi equation
\[
\partial_t \phi_t + \frac12 |\nabla\phi_t|^2 = c(t),
\]
with $c(t)$ a time-dependent scalar.

For a functional $F$, the Wasserstein Hessian acts on $v=\nabla\phi$ as
\[
\mathrm{Hess}_{W_2}F(\mu)[v]
    = \nabla\!\left( \partial_v \frac{\delta F}{\delta\mu}(\mu) \right)
    - \nabla v : \nabla^2 \frac{\delta F}{\delta\mu}(\mu),
\]
where $\partial_v$ denotes directional derivative and $A:B := \mathrm{Tr}(A^TB)$.
This enters the second variation inequality governing geodesic convexity.

The curvature tensor $R(u,v)w$ is defined via
\[
R(u,v)w
:= \frac{D}{ds}\frac{D}{dt} w - \frac{D}{dt}\frac{D}{ds}w
\]
for smooth two-parameter families.
In Otto’s geometry this yields a non-trivial curvature operator.  
For instance, along directions $\nabla\phi$, $\nabla\psi$, one obtains the sectional curvature
\begin{multline}
K(\phi,\psi)
    = \frac{1}{\|\nabla\phi\|^2_\mu\|\nabla\psi\|^2_\mu 
      - \langle\nabla\phi,\nabla\psi\rangle_\mu^2}\\
      \int_\Omega \mu\, \mathrm{Tr}\!\left[
         (\nabla^2\phi)(\nabla^2\psi) - (\nabla^2\psi)(\nabla^2\phi)
      \right] dx,
\end{multline}
which is in general indefinite.

An alternative formulation uses the $\Gamma$--$\Gamma_2$ calculus:
\[
\Gamma(f) = |\nabla f|^2,\qquad
\Gamma_2(f) := \frac12\left( L\Gamma(f) - 2\Gamma(f,Lf)\right),
\]
where $L$ is the generator of the gradient flow of a functional $F$.
A curvature-dimension condition $\Gamma_2(f)\ge \lambda \Gamma(f)$ corresponds exactly to $\lambda$-geodesic convexity of $F$ in Wasserstein space.

The same geometry arises from several equivalent constructions:
\begin{itemize}
\item \emph{Metric approach:} length structure induced by $W_2$ and Benamou–Brenier.
\item \emph{Variational approach:} dynamic optimal transport as a kinetic action minimization.
\item \emph{Lagrangian approach:} geodesics obtained by pushing forward $\mu_0$ along McCann’s displacement interpolation $T_t=(1-t)\mathrm{id}+tT$.
\item \emph{Eulerian approach:} tangent bundle defined by continuity equations and minimal-energy velocity fields.
\item \emph{Connection-based approach:} Levi--Civita connection defined from Helmholtz projection; geodesics satisfy parallel transport.
\end{itemize}

All these constructions are mathematically equivalent and jointly define the full Riemannian structure of $(\mathcal{P}_2(\Omega),W_2)$.

To argue the latter geoemtry is more general than the wasserstein one, you can look intot eh following set of calculations that show, gradient flow in the latter can express former gradient flows:
\section{More on Information Geometry}

Let $(X,\mathcal{F},\lambda)$ be a measurable space with reference measure $\lambda$, and let
\[
M = \{p_\theta : \theta \in \Theta\} \subset \mathcal{P}(X)
\]
be a smooth $d$-dimensional statistical model, where $\Theta \subset \mathbb{R}^d$ is open and
\[
p_\theta(x) = p(x;\theta)
\]
is a smooth family of probability densities with respect to $\lambda$. The \emph{score} is
\[
\partial_i \log p_\theta(x) := \frac{\partial}{\partial \theta^i} \log p(x;\theta), \qquad i=1,\dots,d.
\]
The \emph{Fisher information metric} on $M$ is defined by
\[
g_{ij}(\theta)
:= \int_X \partial_i \log p_\theta(x)\,\partial_j \log p_\theta(x)\,p_\theta(x)\,\mathrm{d}\lambda(x),
\]
turning $(M,g)$ into a finite-dimensional Riemannian manifold. When the FR metric on $\mathcal{P}(X)$ is defined as above, its pullback along the embedding
\[
\Theta \ni \theta \mapsto p_\theta \in \mathcal{P}(X)
\]
coincides with the Fisher information matrix $g_{ij}(\theta)$; in other words, the infinite-dimensional Fisher--Rao metric is compatible with the classical information-geometric construction on parametric families.

Historically, this Riemannian structure was introduced independently by Rao and others and later characterized by Chentsov as the \emph{unique} (up to a constant factor) Riemannian metric on the simplex that is invariant under Markov morphisms and can be obtained as the quadratic term in the Taylor expansion of any standard $f$-divergence between nearby models.

A distinctive feature of information geometry is that the Riemannian manifold $(M,g)$ is enhanced by a pair of torsion-free affine connections $(\nabla,\nabla^\ast)$ that are \emph{dual} with respect to $g$, or more generally by a one-parameter family $\{\nabla^{(\alpha)}\}_{\alpha\in\mathbb{R}}$ of $\alpha$-connections which are mutually dual. 

Let $\nabla$ and $\nabla^\ast$ be torsion-free affine connections on $M$ with Christoffel symbols $\Gamma^k_{ij}$ and $(\Gamma^\ast)^k_{ij}$ in local coordinates $\theta$. They are said to be dual with respect to $g$ if for all vector fields $X,Y,Z$ on $M$,
\begin{equation}
X\big(g(Y,Z)\big) = g(\nabla_X Y,Z) + g(Y,\nabla^\ast_X Z).
\label{eq:dual-connections}
\end{equation}
Equivalently, in coordinates this reads
\begin{equation}
\partial_i g_{jk} = \Gamma^\ell_{ij} g_{\ell k} + (\Gamma^\ast)^\ell_{ik} g_{j\ell}.
\label{eq:dual-connections-coordinates}
\end{equation}

The Levi--Civita connection $\nabla^{(0)}$ of $(M,g)$ is the unique torsion-free \emph{metric} connection satisfying
\begin{equation}
\partial_i g_{jk} = \Gamma^{(0)\,\ell}_{\;\;ij} g_{\ell k} + \Gamma^{(0)\,\ell}_{\;\;ik} g_{j\ell},
\label{eq:levi-civita}
\end{equation}
so it corresponds to the self-dual case $\nabla = \nabla^\ast = \nabla^{(0)}$.

In information geometry one introduces the totally symmetric \emph{Amari--Chentsov tensor}
\[
T_{ijk}(\theta)
:= \int_X \partial_i \log p_\theta(x)\,\partial_j \log p_\theta(x)\,\partial_k \log p_\theta(x)\,p_\theta(x)\,\mathrm{d}\lambda(x),
\]
and its $(1,2)$-version
\[
T^k_{\;ij} := g^{k\ell} T_{ij\ell}.
\]
The $\alpha$-connections are then defined by
\begin{equation}
\Gamma^{(\alpha)\,k}_{\;\;\;ij}
= \Gamma^{(0)\,k}_{\;\;\;ij} + \frac{\alpha}{2}\,T^k_{\;ij},
\label{eq:alpha-connection}
\end{equation}
and one checks that $\nabla^{(\alpha)}$ and $\nabla^{(-\alpha)}$ are dual with respect to $g$ in the sense of \eqref{eq:dual-connections}. The three most important cases are\text{ (mixture connection)}, \text{ (Levi--Civita / Fisher--Rao connection)}\text{ (exponential connection)}.

\begin{definition}[Statistical manifold]
A \emph{statistical manifold} in the sense of information geometry is a quadruple
\[
(M,g,\nabla,\nabla^\ast)
\]
consisting of a smooth manifold $M$, a Riemannian metric $g$, and a pair of torsion-free affine connections $(\nabla,\nabla^\ast)$ that are dual with respect to $g$ as in \eqref{eq:dual-connections}. The special case $(M,g,\nabla^{(\alpha)},\nabla^{(-\alpha)})$ is called an \emph{$\alpha$-statistical manifold}.
\end{definition}

Geometrically, the dualistic structure $(g,\nabla,\nabla^\ast)$ allows one to decompose curvature and convexity into ``primal'' and ``dual'' contributions, and it underlies many higher-order asymptotic results in statistics (efficiency, bias, and so on).

Given a connection $\nabla^{(\alpha)}$, a curve $\theta(t)$ in parameter space is a $\nabla^{(\alpha)}$-geodesic if its velocity vector field is parallel along the curve:
\[
\nabla^{(\alpha)}_{\dot\theta(t)} \dot\theta(t) = 0,
\]
or, in local coordinates,
\begin{equation}
\ddot\theta^k(t)
+ \Gamma^{(\alpha)\,k}_{\;\;\;ij}(\theta(t))\,\dot\theta^i(t)\,\dot\theta^j(t)
= 0.
\label{eq:alpha-geodesic}
\end{equation}
In many important models (in particular exponential families), the $\alpha=\pm1,0$ connections have especially simple geodesics when expressed in suitable affine coordinates.

\medskip
\noindent
\emph{Mixture geodesics $(\alpha=-1)$.}
In mixture coordinates, a $(-1)$-geodesic between densities $\rho_0,\rho_1$ is simply the linear interpolation
\[
\rho^{\mathrm{mix}}_t = (1-t)\,\rho_0 + t\,\rho_1,\qquad t\in[0,1].
\]
On a finite-dimensional exponential family, mixture geodesics correspond to affine lines in the expectation-parameter coordinates. In this sense, the mixture connection encodes the affine structure of the simplex under convex combination.

\medskip
\noindent
\emph{Exponential geodesics $(\alpha=+1)$.}
In natural (canonical) parameters, a $(+1)$-geodesic is linear in the log-density. At the level of densities this gives the normalized geometric interpolation
\[
\rho^{\mathrm{exp}}_t
= \frac{\rho_0^{1-t}\,\rho_1^{t}}{\displaystyle \int_X \rho_0^{1-t}(x)\,\rho_1^{t}(x)\,\mathrm{d}\lambda(x)},
\qquad t\in[0,1].
\]
On an exponential family, exponential geodesics are affine lines in the natural parameter $\theta$, and their image in density space is obtained by exponentiating and normalizing.

\medskip
\noindent
\emph{Fisher--Rao / Levi--Civita geodesics $(\alpha=0)$.}
For the Fisher--Rao connection, geodesics are the usual Riemannian geodesics of $(M,g)$. In the infinite-dimensional ambient space $(\mathcal{M}_+(\Omega),d_{\mathrm{FR}})$, the map $\mu\mapsto \sqrt{\mu}$ identifies the manifold with a convex cone in the Hilbert space $L^2(\Omega)$. The FR geodesic between $\mu_0$ and $\mu_1$ is simply the image of the straight line between $\sqrt{\mu_0}$ and $\sqrt{\mu_1}$ in $L^2(\Omega)$, restricted to the cone:
\[
\sqrt{\mu_t} = (1-t)\sqrt{\mu_0} + t\sqrt{\mu_1},\qquad
\mu_t = \big((1-t)\sqrt{\mu_0} + t\sqrt{\mu_1}\big)^2.
\]
On a normalized probability simplex, this corresponds to great-circle arcs on the unit sphere of $L^2$, and one can compute curvature tensors explicitly for specific parametric families (e.g.\ location-scale Gaussians) to see that many statistical manifolds have negative sectional curvature even though the ambient FR cone is flat; this is the source of the ``negative curvature of statistical manifolds'' versus ``flat Fisher--Rao cone'' contrast.

The Fisher--Rao construction above is inherently infinite-dimensional: $\mathcal{P}(\Omega)$ (or $\mathcal{M}_+(\Omega)$) is modelled on suitable function spaces (e.g.\ $L^2$ or Orlicz spaces), and the map $\mu\mapsto \sqrt{\mu}$ realizes it as a (nonlinear) submanifold of $L^2(\Omega)$. In this setting, tangent vectors are identified with square-integrable functions $\varphi$ with zero mean, and the FR inner product is simply
\[
\langle \varphi_1,\varphi_2\rangle_\mu^{\mathrm{FR}}
= \int_\Omega \varphi_1(x)\,\varphi_2(x)\,\mathrm{d}\mu(x).
\]
Restricting this infinite-dimensional structure to a finite-dimensional statistical model $M=\{p_\theta\}$ recovers the Fisher information metric, and the dualistic structure $(g,\nabla^{(\alpha)},\nabla^{(-\alpha)})$ lifts to suitable infinite-dimensional settings.

One rigorous approach, due to Pistone and Sempi, constructs an infinite-dimensional exponential statistical manifold on the space of all probability measures equivalent to a fixed reference measure, with charts given by centered log-densities and with a natural extension of the FR metric and $\alpha$-connections. This provides a genuine Banach (or Fréchet) manifold structure on $\mathcal{P}(\Omega)$ in which mixture and exponential geodesics, as well as FR geodesics, can be treated on the same footing, and where the dualistic information-geometric calculus extends beyond parametric models.

In summary, information geometry equips spaces of probability measures with a Riemannian metric (Fisher--Rao), a family of dual affine connections ($\alpha$-connections), and the associated mixture, exponential, and FR geodesics. This dualistic structure is compatible with both finite-dimensional parametric models and infinite-dimensional manifolds of measures, and it will be crucial later when we relate curvature, spectral properties of Markov semigroups, and stability of PDE/SDE flows.

\section{More on HK geometry}
From the Riemannian point of view, $(\mathcal{M}_+(\Omega),d_{\mathrm{WFR}})$ is a formal infinite-dimensional Riemannian manifold whose tangent space at $\mu$ is identified with potentials $\phi$ modulo constants, endowed with the inner product
\[
\langle \phi_1,\phi_2\rangle_\mu^{\mathrm{WFR}}
= \int_\Omega \big(\phi_1\phi_2 + \nabla\phi_1\cdot\nabla\phi_2\big)\,\mathrm{d}\mu.
\]
The associated Levi--Civita connection is obtained by projecting the time derivative of a time-dependent potential plus its convective derivative (arising from $v_t=\nabla\phi_t$) back onto this tangent space, exactly as in the Otto calculus but with an additional reaction component. Geodesics in this geometry are characterized variationally as minimizers of the action $\mathcal{A}$, and infinitesimally by the condition that their velocity potentials are parallel transported by this connection.

Geometrically, the WFR (Hellinger--Kantorovich) space can also be represented as a cone over a suitable Wasserstein base space: informally, one can rewrite a positive measure as a “mass” (radial variable) times a probability distribution (angular variable). The radial part evolves according to a Fisher--Rao-type dynamics, while the angular part follows a Wasserstein-type displacement. In this sense, purely radial geodesics reproduce Fisher--Rao geodesics, purely angular geodesics reproduce Wasserstein geodesics, and general WFR geodesics involve simultaneous changes in both mass and spatial distribution.

\medskip

\section{Markov Semigroups, Spectral Gaps, and Curvature via Carr\'e du Champ}
\label{sec:markov-semigroups}

\subsection{Markov semigroups and infinitesimal generators}

Let $(\Omega,\mathcal{F},(\mathcal{F}_t)_{t\ge 0},\mathbb{P})$ be a filtered probability space and
$(X_t)_{t\ge 0}$ a time-homogeneous Markov process on $\mathbb{R}^d$.
Its Markov semigroup $(P_t)_{t\ge 0}$ acts on suitable test functions $\varphi:\mathbb{R}^d\to\mathbb{R}$ by
\begin{equation}
  (P_t\varphi)(x) \;:=\; \mathbb{E}\!\left[\varphi(X_t)\,\big|\,X_0=x\right].
  \label{eq:Pt-def}
\end{equation}
The semigroup property $P_{t+s}=P_tP_s$ follows from the Markov property.

The (infinitesimal) generator $\mathcal{L}$ is defined on its domain $\mathsf{D}(\mathcal{L})$ by the strong limit
\begin{equation}
  \mathcal{L}\varphi
  \;:=\;
  \lim_{t\downarrow 0}\frac{P_t\varphi-\varphi}{t},
  \qquad \varphi\in\mathsf{D}(\mathcal{L}).
  \label{eq:generator-def}
\end{equation}
Formally one may write $P_t=\exp(t\mathcal{L})$ and the backward Kolmogorov equation becomes
\begin{equation}
  \partial_t (P_t\varphi) \;=\; \mathcal{L}(P_t\varphi) \;=\; P_t(\mathcal{L}\varphi),
  \qquad P_0\varphi=\varphi.
  \label{eq:backward-kolmogorov}
\end{equation}

When the law of $X_t$ has a density $p_t$ w.r.t.\ Lebesgue measure, the forward evolution is governed by the adjoint
$\mathcal{L}^*$ (in the $\langle\cdot,\cdot\rangle_{L^2(dx)}$ pairing):
\begin{equation}
  \partial_t p_t \;=\; \mathcal{L}^* p_t,
  \qquad
  \int \varphi(x)\,\mathcal{L}^*p(x)\,dx
  \;=\;
  \int (\mathcal{L}\varphi)(x)\,p(x)\,dx.
  \label{eq:forward-kolmogorov}
\end{equation}

\subsection{The OU/Langevin diffusion as the canonical example}

A guiding example throughout this paper is the Langevin (Ornstein--Uhlenbeck as the quadratic case) diffusion
\begin{equation}
  dX_t \;=\; -\nabla V(X_t)\,dt + \sqrt{2}\,dB_t,
  \label{eq:langevin-sde}
\end{equation}
for a potential $V:\mathbb{R}^d\to\mathbb{R}$.
Its generator acting on smooth compactly supported $\varphi$ is
\begin{equation}
  \mathcal{L}\varphi \;=\; \Delta\varphi - \langle\nabla V,\nabla\varphi\rangle,
  \label{eq:langevin-generator}
\end{equation}
and the invariant probability measure (when normalizable) is
\begin{equation}
  \pi(dx) \;\propto\; e^{-V(x)}\,dx.
  \label{eq:gibbs}
\end{equation}
For the OU process, $V(x)=\tfrac{1}{2}\|x\|^2$ and $\pi$ is Gaussian.

It is often advantageous to switch from Lebesgue to the invariant measure and work in $L^2(\pi)$.
In particular, rather than evolving the Lebesgue-density $p_t$, one may evolve the relative density
$\rho_t := p_t/\pi$, in which case (under reversibility assumptions made below) the forward equation takes the
simple form
\begin{equation}
  \partial_t \rho_t \;=\; \mathcal{L}\rho_t.
  \label{eq:rho-forward}
\end{equation}

\subsection{Reversibility, symmetry, and Dirichlet forms}

A probability measure $\pi$ is stationary for $(P_t)$ if $\int P_t\varphi\,d\pi=\int \varphi\,d\pi$ for all $t\ge 0$.
The semigroup is \emph{reversible} w.r.t.\ $\pi$ if
\begin{equation}
  \int f\,(P_t g)\,d\pi \;=\; \int (P_t f)\,g\,d\pi,
  \qquad \forall f,g \in L^2(\pi),\ \forall t\ge 0.
  \label{eq:reversible}
\end{equation}
Equivalently, $P_t$ (and hence $\mathcal{L}$) is self-adjoint in $L^2(\pi)$, and one expects a real spectrum.
A basic consequence is the $L^2(\pi)$ contraction inequality
\begin{equation}
  \|P_t f\|_{L^2(\pi)} \;\le\; \|f\|_{L^2(\pi)},
  \qquad t\ge 0,
  \label{eq:L2-contraction}
\end{equation}
which can be obtained from Jensen and stationarity.

Given a (symmetric) generator $\mathcal{L}$, the \emph{carr\'e du champ} operator $\Gamma$ is defined by
\begin{equation}
  \Gamma(f,g)
  \;:=\;
  \frac{1}{2}\Big(\mathcal{L}(fg) - f\,\mathcal{L}g - g\,\mathcal{L}f\Big),
  \qquad \Gamma(f):=\Gamma(f,f).
  \label{eq:gamma-def}
\end{equation}
The associated Dirichlet energy is
\begin{equation}
  \mathcal{E}(f,g) \;:=\; \int \Gamma(f,g)\,d\pi.
  \label{eq:dirichlet-energy}
\end{equation}
Under reversibility, one has the fundamental integration-by-parts identity
\begin{equation}
  \int f\,(-\mathcal{L})g\,d\pi
  \;=\;
  \int \Gamma(f,g)\,d\pi
  \;=\;
  \mathcal{E}(f,g),
  \label{eq:ibp}
\end{equation}
so that $-\mathcal{L}$ is positive semidefinite on $L^2(\pi)$.

For the Langevin generator \eqref{eq:langevin-generator}, a direct computation yields the canonical identity
\begin{equation}
  \Gamma(f) \;=\; \|\nabla f\|^2.
  \label{eq:gamma-langevin}
\end{equation}

\subsection{Spectral gap, Poincar\'e inequality, and mixing rates}

The spectral gap of $-\mathcal{L}$ (in $L^2(\pi)$) governs exponential decay to equilibrium in $L^2$.
A standard sufficient condition is the Poincar\'e inequality: there exists $\lambda_{\mathrm{PI}}>0$ such that
\begin{equation}
  \mathrm{Var}_\pi(f)
  \;:=\;
  \int \Big(f-\textstyle\int f\,d\pi\Big)^2\,d\pi
  \;\le\;
  \frac{1}{\lambda_{\mathrm{PI}}}\int \Gamma(f)\,d\pi,
  \qquad \forall f \in \mathsf{D}(\mathcal{E}).
  \label{eq:poincare}
\end{equation}
When \eqref{eq:poincare} holds, one obtains exponential decay of variance along the semigroup:
\begin{equation}
  \mathrm{Var}_\pi(P_t f)
  \;\le\;
  e^{-2\lambda_{\mathrm{PI}} t}\,\mathrm{Var}_\pi(f),
  \qquad t\ge 0,
  \label{eq:variance-decay}
\end{equation}
which is a precise, operator-theoretic notion of \emph{mixing rate}.

A stronger inequality is the logarithmic Sobolev inequality (LSI):
there exists $C_{\mathrm{LSI}}>0$ such that for all densities $\mu$ w.r.t.\ $\pi$,
\begin{equation}
  \mathrm{KL}(\mu\|\pi)
  \;\le\;
  C_{\mathrm{LSI}}\ \mathrm{FI}(\mu\|\pi),
  \qquad
  \mathrm{FI}(\mu\|\pi)
  \;:=\;
  \mathbb{E}_\mu\!\big[\|\nabla \log(\mu/\pi)\|^2\big].
  \label{eq:lsi}
\end{equation}
LSI implies exponential decay of relative entropy (and hence, under mild conditions, convergence in total variation).

\subsection{Iterated carr\'e du champ and curvature-dimension}

The \emph{iterated carr\'e du champ} $\Gamma_2$ is defined by
\begin{equation}
  \Gamma_2(f,g)
  \;:=\;
  \frac{1}{2}\Big(
    \mathcal{L}\Gamma(f,g) - \Gamma(f,\mathcal{L}g) - \Gamma(g,\mathcal{L}f)
  \Big),
  \qquad \Gamma_2(f):=\Gamma_2(f,f).
  \label{eq:gamma2-def}
\end{equation}
Conceptually, $\Gamma$ appears when differentiating Lyapunov functionals (e.g.\ $\chi^2$ or KL) once in time,
while $\Gamma_2$ appears upon differentiating a second time, and it is the correct analytic object that encodes
\emph{curvature} information of the diffusion (in the Bakry--\'Emery sense).

A diffusion Markov semigroup is said to satisfy the Bakry--\'Emery criterion with constant $\alpha>0$ if
\begin{equation}
  \Gamma_2(f) \;\ge\; \alpha\,\Gamma(f),
  \qquad \forall f.
  \label{eq:BE}
\end{equation}
This is also written as the curvature-dimension condition $\mathrm{CD}(\alpha,\infty)$.
In the Langevin case, $\mathrm{CD}(\alpha,\infty)$ is equivalent to $\alpha$-strong convexity of $V$:
\begin{equation}
  \nabla^2 V(x) \succeq \alpha I
  \quad\Longleftrightarrow\quad
  \mathrm{CD}(\alpha,\infty).
  \label{eq:strong-convexity}
\end{equation}
The significance for sampling is that $\mathrm{CD}(\alpha,\infty)$ yields functional inequalities (in particular, LSI)
with constants controlled by $1/\alpha$, hence quantitative convergence rates.

\begin{table*}[t]
\centering
\renewcommand{\arraystretch}{1.35}
\setlength{\tabcolsep}{6pt}
\begin{tabular}{p{0.23\linewidth} p{0.73\linewidth}}
\hline
\textbf{Concept} & \textbf{OU specialization (and the general semigroup identities it instantiates)}\\
\hline

OU SDE &
\(
\mathrm{d}X_t = -\nabla V(X_t)\,\mathrm{d}t + \sqrt{2}\,\mathrm{d}B_t,
\qquad
V(x)=\frac{\alpha}{2}\|x\|^2\ (\alpha>0),
\)
so \(\nabla V(x)=\alpha x\) and
\(
\mathrm{d}X_t=-\alpha X_t\,\mathrm{d}t+\sqrt{2}\,\mathrm{d}B_t.
\)

\\

Markov semigroup \((P_t)_{t\ge 0}\) &
\(
(P_t f)(x) := \mathbb{E}\!\left[f(X_t)\mid X_0=x\right]
\)
(time-homogeneous; semigroup property \(P_sP_t=P_{s+t}\)).  

\\

Infinitesimal generator \(\mathcal{L}\) &
\(
\mathcal{L}f := \lim_{t\downarrow 0}\frac{P_t f-f}{t}
\)
(on the appropriate domain).  
\\

Explicit generator (OU/Langevin form) &
For \(f\in C^2(\mathbb{R}^d)\),
\[
(\mathcal{L}f)(x)
= \Delta f(x)-\langle \nabla V(x), \nabla f(x)\rangle
= \Delta f(x)-\alpha \langle x,\nabla f(x)\rangle .
\]
(This is exactly the Langevin generator formula specialized to quadratic \(V\).) 

\\

Kolmogorov (semigroup) evolution for observables &
For all \(t\ge 0\),
\[
\partial_t P_t f = \mathcal{L}P_t f = P_t\mathcal{L}f ,
\]
i.e. \(u_t:=P_t f\) solves \(\partial_t u_t=\mathcal{L}u_t\). 

\\

Forward (Fokker--Planck) equation &
Writing \(\pi_t\) for the law/density of \(X_t\) (w.r.t. Lebesgue),
\[
\partial_t \pi_t = \mathcal{L}^\ast \pi_t ,
\]
where \(\mathcal{L}^\ast\) is the \(L^2(\mathfrak{m})\)-adjoint of \(\mathcal{L}\).  
\\

Stationary distribution \(\pi\) &
\(\pi\) is stationary iff \(\mathcal{L}^\ast \pi=0\).  
For OU with \(V(x)=\frac{\alpha}{2}\|x\|^2\),
\[
\pi(x)\ \propto\ \exp(-V(x)) \ =\ \exp\!\Big(-\frac{\alpha}{2}\|x\|^2\Big),
\quad\text{i.e. } \pi=\mathcal{N}\!\left(0,\alpha^{-1}I_d\right).
\]

\\

Reversibility (w.r.t. \(\pi\)) &
The OU/Langevin semigroup is reversible w.r.t. \(\pi\), so \(\mathcal{L}\) is symmetric on \(L^2(\pi)\).  

\\

Carr\'e du champ \(\Gamma\) (your \(T\)) &
Define
\[
\Gamma(f,g):=\frac12\big(\mathcal{L}(fg)-f\,\mathcal{L}g-g\,\mathcal{L}f\big).
\]

For the diffusion generator above (in particular OU),
\[
\Gamma(f,g)=\langle \nabla f,\nabla g\rangle,
\qquad
\Gamma(f,f)=\|\nabla f\|^2 .
\]

\\

Dirichlet energy / Dirichlet form \(\mathcal{E}\) &
\[
\mathcal{E}(f,g):=\int \Gamma(f,g)\,\mathrm{d}\pi,
\qquad
\mathcal{E}(f,f)=\int \|\nabla f\|^2\,\mathrm{d}\pi .
\]

\\

Fundamental integration-by-parts identity &
For reversible semigroups,
\[
\int f\,(-\mathcal{L})g\,\mathrm{d}\pi
=
\int \Gamma(f,g)\,\mathrm{d}\pi
=
\mathcal{E}(f,g),
\quad\text{hence } -\mathcal{L}\succeq 0.
\]

\\

Iterated carr\'e du champ \(\Gamma_2\) (your \(T_2\)) &
Define
\[
\Gamma_2(f,g)
:=\frac12\Big(\mathcal{L}\Gamma(f,g)-\Gamma(f,\mathcal{L}g)-\Gamma(g,\mathcal{L}f)\Big).
\]

(For OU, \(\Gamma_2\) can be computed explicitly and is the object controlling curvature-dimension/Bakry--\'Emery estimates, but the defining identity above is the canonical “reader convenience” form.)

\\
\hline
\end{tabular}
\caption{Ornstein--Uhlenbeck (OU) process as the quadratic-potential Langevin diffusion: semigroup \(P_t\), generator \(\mathcal{L}\), forward/backward PDEs, and the quadratic forms \(T=\Gamma\), \(T_2=\Gamma_2\), and Dirichlet energy \(\mathcal{E}\).}
\label{tab:ou_semigroup_generator_gamma}
\end{table*}

\subsection{How Fisher--Rao reweighting can affect the spectral gap: a precise operator viewpoint}
\label{subsec:fr-spectral-gap}

In this subsection we make precise, at the level of generators and quadratic
forms, the sense in which adding a Fisher--Rao (FR) \emph{reaction/reweighting}
term on top of an Ornstein--Uhlenbeck (OU) diffusion can change the spectral
properties that govern convergence rates. We emphasize from the outset that the
resulting dynamics is no longer a linear Markov semigroup on densities; rather,
it is a Feynman--Kac (FK) evolution with normalization, or equivalently a
mean-field interacting particle system. Consequently, the classical notion of
\emph{spectral gap of a Markov generator} does not apply verbatim. Nevertheless,
there is a natural and rigorous way to measure the additional dissipation
induced by the FR term via bilinear forms and variance decay identities.

Let $\pi$ denote the invariant Gaussian measure of the OU process
\(
dX_t=-\alpha X_t\,dt+\sqrt{2}\,dB_t
\)
(with $\alpha>0$), and let $\mathcal{L}$ be its (symmetric) generator on
$L^2(\pi)$,
\[
\mathcal{L}f=\Delta f-\alpha\langle x,\nabla f\rangle.
\]
The OU semigroup $(P_t)_{t\ge 0}$ is reversible w.r.t.\ $\pi$, and the spectral
gap inequality (Poincar\'e inequality) asserts that there exists
$\lambda_{\mathrm{OU}}>0$ such that
\begin{equation}
\mathrm{Var}_\pi(f)
\;\le\;
\frac{1}{\lambda_{\mathrm{OU}}}\,\mathcal{E}(f,f),
\qquad
\mathcal{E}(f,f)=\int \Gamma(f)\,d\pi=\int \|\nabla f\|^2\,d\pi,
\label{eq:ou-poincare}
\end{equation}
for all $f$ in the Dirichlet domain with $\int f\,d\pi=0$. Equivalently,
\begin{equation}
\int f\,(-\mathcal{L})f\,d\pi
\;\ge\;
\lambda_{\mathrm{OU}}\int f^2\,d\pi,
\qquad f\perp 1,
\label{eq:ou-gap-form}
\end{equation}
which yields exponential $L^2(\pi)$ convergence:
\(
\|P_t f\|_{L^2(\pi)}\le e^{-\lambda_{\mathrm{OU}} t}\|f\|_{L^2(\pi)}.
\)

Let $g:\mathbb{R}^d\to\mathbb{R}$ be a measurable potential (in applications,
$g$ may depend on $t$ and on the current density through a corrector; here we
freeze $g$ to isolate the spectral mechanism). Consider the \emph{normalized}
FK evolution of densities w.r.t.\ Lebesgue,
\begin{equation}
\partial_t p_t
=
\mathcal{L}^* p_t
+
p_t\Bigl(g-\mathbb{E}_{p_t}[g]\Bigr),
\qquad
\mathbb{E}_{p_t}[g]=\int g(x)p_t(x)\,dx,
\label{eq:fk-normalizeds}
\end{equation}
which is exactly the OU transport--diffusion plus an FR-type reweighting term.
The centering by $\mathbb{E}_{p_t}[g]$ enforces mass conservation
$\int p_t=1$. This dynamics is no longer linear in $p_t$, hence it does not
define a Markov semigroup on densities.

A convenient way to compare with the OU spectral picture is to switch to the
\emph{relative density} $\rho_t:=\frac{p_t}{\pi}$, for which the baseline OU
evolution becomes $\partial_t \rho_t=\mathcal{L}\rho_t$ (cf.\ the semigroup
preliminaries). In these coordinates, \eqref{eq:fk-normalizeds} reads
\begin{equation}
\partial_t \rho_t
=
\mathcal{L}\rho_t
+
\rho_t\Bigl(g-\mathbb{E}_{\rho_t\pi}[g]\Bigr),
\qquad
\mathbb{E}_{\rho_t\pi}[g]=\int g\,\rho_t\,d\pi.
\label{eq:fk-rho}
\end{equation}

A natural $L^2(\pi)$ distance from equilibrium is the $\chi^2$ functional
\(
\chi^2(p_t\|\pi)=\int (\rho_t-1)^2\,d\pi=\mathrm{Var}_\pi(\rho_t)
\)
(since $\int \rho_t\,d\pi=1$). Differentiating along \eqref{eq:fk-rho} and using
the integration-by-parts identity
\(
\int f\,\mathcal{L}f\,d\pi=-\int \Gamma(f)\,d\pi
\),
we obtain the exact identity
\begin{multline}
\frac{1}{2}\frac{d}{dt}\mathrm{Var}_\pi(\rho_t)
=
\int (\rho_t-1)\,\partial_t\rho_t\,d\pi
\nonumber\\
=
\underbrace{\int (\rho_t-1)\,\mathcal{L}\rho_t\,d\pi}_{-\int \Gamma(\rho_t)\,d\pi}
\;+\;
\int (\rho_t-1)\,\rho_t\Bigl(g-\mathbb{E}_{\rho_t\pi}[g]\Bigr)\,d\pi
\nonumber\\
=
-\int \Gamma(\rho_t)\,d\pi
\;+\;
\int \rho_t(\rho_t-1)\,g\,d\pi
\;-\;
\mathbb{E}_{\rho_t\pi}[g]\int \rho_t(\rho_t-1)\,d\pi .
\label{eq:var-identity}
\end{multline}
Since $\int \rho_t(\rho_t-1)\,d\pi=\mathrm{Var}_\pi(\rho_t)$, the last term is
explicit:
\begin{equation}
\frac{1}{2}\frac{d}{dt}\mathrm{Var}_\pi(\rho_t)
=
-\int \Gamma(\rho_t)\,d\pi
+
\int \rho_t(\rho_t-1)\,g\,d\pi
-
\mathbb{E}_{\rho_t\pi}[g]\;\mathrm{Var}_\pi(\rho_t).
\label{eq:var-identity2}
\end{equation}
Equation \eqref{eq:var-identity2} is rigorous and shows precisely how the
reweighting term alters the baseline OU dissipation $-\int\Gamma(\rho_t)d\pi$ by
two additional contributions, both governed by correlations between $\rho_t$ and
$g$.

To connect \eqref{eq:var-identity2} to spectral gaps, we linearize the flow
around equilibrium. Let $\rho_t=1+\varepsilon h_t$ with $\int h_t\,d\pi=0$ and
$\varepsilon\ll 1$. Expanding \eqref{eq:fk-rho} to first order in $\varepsilon$
yields the linearized evolution
\begin{equation}
\partial_t h_t
=
\mathcal{L} h_t
+
\Bigl(g-\mathbb{E}_\pi[g]\Bigr)\,h_t,
\label{eq:linearized}
\end{equation}
since the centering term contributes only through $\mathbb{E}_\pi[g]$ at leading
order. Define the (generally non-symmetric) linear operator
\begin{equation}
\mathcal{A} := \mathcal{L} + \bigl(g-\mathbb{E}_\pi[g]\bigr)\,\mathrm{Id}.
\label{eq:A-def-fr}
\end{equation}
In the reversible OU setting, $\mathcal{L}$ is self-adjoint on $L^2(\pi)$, and
multiplication by $(g-\mathbb{E}_\pi[g])$ is also self-adjoint. Thus $\mathcal{A}$
is self-adjoint on $L^2(\pi)$ and the linearized dynamics admits an $L^2(\pi)$
spectral decomposition. In particular, if
\begin{equation}
-\sup\bigl\{\langle h,\mathcal{A}h\rangle_{L^2(\pi)} : \|h\|_{L^2(\pi)}=1,\ h\perp 1\bigr\}
\;=\;
\lambda_{\mathrm{eff}}
\;>\; 0,
\label{eq:eff-gap}
\end{equation}
then the linearized perturbations decay exponentially:
\(
\|h_t\|_{L^2(\pi)}\le e^{-\lambda_{\mathrm{eff}}t}\|h_0\|_{L^2(\pi)}.
\)
Comparing \eqref{eq:eff-gap} with the OU gap \eqref{eq:ou-gap-form}, we see that
the FR potential effectively modifies the Rayleigh quotient by adding
\(
\int (g-\mathbb{E}_\pi[g])\,h^2\,d\pi.
\)
Hence, \emph{whenever the potential is such that}
\begin{equation}
\int (g-\mathbb{E}_\pi[g])\,h^2\,d\pi
\;\le\;
-\,c\,\|h\|_{L^2(\pi)}^2
\quad \text{for all } h\perp 1,
\label{eq:neg-potential-condition}
\end{equation}
for some $c>0$, the effective gap improves to
\(
\lambda_{\mathrm{eff}}\ge \lambda_{\mathrm{OU}}+c.
\)
Condition \eqref{eq:neg-potential-condition} holds, for instance, if
$g-\mathbb{E}_\pi[g]\le -c$ $\pi$-a.s.\ (a strong sufficient condition), or more
generally if $g$ is negative on the dominant modes of $-\mathcal{L}$.

The preceding argument makes two points precise:
\begin{enumerate}
\item The FR reweighting term changes the \emph{dissipation identity}
\eqref{eq:var-identity2} by adding terms controlled by correlations with the
potential $g$.
\item After linearization around equilibrium, the FR term appears as a
self-adjoint ``killing/anti-killing'' perturbation of the OU generator
\eqref{eq:A-def-fr}, and can therefore increase or decrease the effective
spectral gap depending on the sign/structure of $g$ through the Rayleigh
quotient \eqref{eq:eff-gap}.
\end{enumerate}
In particular, the statement ``FR improves the spectral gap'' is \emph{not}
automatic: it requires that the induced potential acts as additional
dissipation on the relevant slow modes.

In the applications of this paper, $\psi_t$ is not an arbitrary fixed potential
but is induced by geometric correction terms (e.g.\ mixture/exponential/OT
interpolations) and may depend on the evolving density. This places the
resulting dynamics outside the scope of classical linear Markov semigroup
theory: the flow is nonlinear, and its implementation relies on interacting
particles with resampling/jump mechanisms. A systematic theory connecting such
normalized Feynman--Kac evolutions to \emph{quantitative} improvements of mixing
rates via spectral gaps (or suitable nonlinear analogues) appears to be
underdeveloped. Establishing sharp conditions under which the induced FR/FK
potentials enlarge an effective gap and yield provably faster convergence is an
important direction for future research.

\section{Three elementary lemmas: drift, diffusion, and Fisher--Rao rates}

In this section we isolate three simple but important identities that
clarify how the classical drift and diffusion terms appearing in the
Fokker--Planck equation can be re-expressed in terms of (i) a pure
continuity equation with a suitable velocity field, and (ii) Fisher--Rao
reaction equations of the form
\[
\partial_t \mu_t = \psi_t\,\mu_t,
\]
for an explicitly computable rate function $\psi_t$.  These identities will
be the basic algebraic tools that allow us, later on, to express
diffusion-model dynamics in Fisher--Rao language and to identify the
correct Feynman--Kac weights.

Throughout we work on $\mathbb{R}^d$ and assume that for each $t$,
$\mu_t$ is a strictly positive, smooth density with sufficient decay at
infinity so that all integrations by parts below are justified.  We write
$\mu_t(x)$ simply as $\mu_t$ when no confusion can arise.

\subsection{Diffusion as drift}

We first show that pure diffusion can be rewritten exactly as a continuity
equation with a suitable (state- and time-dependent) velocity field.  This
gives a first indication that diffusion is not fundamentally different
from drift at the level of the PDE; rather, it is a very specific choice
of drift depending on the current density.

\begin{lemma}[Diffusion can be written as drift]
\label{lem:diffusion-as-drift}
Let $(\mu_t)_{t\ge0}$ solve the heat equation
\begin{equation}
  \partial_t \mu_t
  \;=\;
  \frac{\sigma_t^2}{2}\,\Delta \mu_t,
  \qquad t \ge 0.
  \label{eq:heat}
\end{equation}
Define the vector field
\begin{equation}
  v_t(x)
  \;:=\;
  -\frac{\sigma_t^2}{2}\,\nabla \log \mu_t(x),
  \qquad x \in \mathbb{R}^d.
  \label{eq:velocity-from-diffusion}
\end{equation}
Then $(\mu_t,v_t)$ satisfy the continuity equation
\begin{equation}
  \partial_t \mu_t + \nabla\!\cdot(\mu_t v_t) = 0,
  \label{eq:continuity-from-diffusion}
\end{equation}
and conversely, any strictly positive solution of
\eqref{eq:continuity-from-diffusion} with $v_t$ given by
\eqref{eq:velocity-from-diffusion} also solves the heat equation
\eqref{eq:heat}.
\end{lemma}

\begin{proof}
Using the definition \eqref{eq:velocity-from-diffusion}, we compute
\[
  \mu_t(x)\, v_t(x)
  =
  -\frac{\sigma_t^2}{2}\,\mu_t(x)\,\nabla \log \mu_t(x)
  =
  -\frac{\sigma_t^2}{2}\,\nabla \mu_t(x),
\]
since $\nabla \log \mu_t = \nabla \mu_t / \mu_t$.  Therefore
\[
  \nabla\!\cdot\big(\mu_t v_t\big)
  =
  -\frac{\sigma_t^2}{2}\,\nabla\!\cdot(\nabla \mu_t)
  =
  -\frac{\sigma_t^2}{2}\,\Delta \mu_t.
\]
Substituting into the continuity equation
$\partial_t \mu_t + \nabla\!\cdot(\mu_t v_t) = 0$ yields
\[
  \partial_t \mu_t
  =
  - \nabla\!\cdot(\mu_t v_t)
  =
  \frac{\sigma_t^2}{2}\,\Delta \mu_t,
\]
which is exactly the heat equation \eqref{eq:heat}.  Conversely, if
$\mu_t$ solves \eqref{eq:heat}, setting $v_t$ as in
\eqref{eq:velocity-from-diffusion} gives the desired continuity equation
\eqref{eq:continuity-from-diffusion}.  This proves the equivalence.
\end{proof}

\begin{remark}
Lemma~\ref{lem:diffusion-as-drift} shows that diffusion can be seen as
a very particular drift, namely the \emph{score-driven} drift
$v_t = -(\sigma_t^2/2)\nabla \log \mu_t$.  In Wasserstein geometry, this
is the velocity field corresponding to the $W_2$--gradient of the
negative entropy functional.
\end{remark}

\subsubsection{Drift as Fisher--Rao reaction}

We next show that the usual drift term, expressed via the continuity
equation, admits an equivalent Fisher--Rao representation
\[
\partial_t \mu_t = \psi_t \,\mu_t,
\]
with an explicit rate $\psi_t$ depending on the drift and on the log-density.
This is the precise sense in which a transport equation can be understood
through the Fisher--Rao lens.

\begin{lemma}[Drift induces a Fisher--Rao rate]
\label{lem:drift-as-FR}
Let $(\mu_t)_{t\ge 0}$ solve the continuity equation
\begin{equation}
  \partial_t \mu_t
  =
  -\nabla\!\cdot\big(\mu_t v_t\big),
  \qquad t \ge 0,
  \label{eq:continuity-drift}
\end{equation}
for a given smooth drift field $v_t : \mathbb{R}^d \to \mathbb{R}^d$.
Assume $\mu_t>0$ everywhere.  Then
\begin{equation}
  \partial_t \mu_t(x)
  =
  \psi_t(x)\,\mu_t(x),
  \qquad
  \psi_t(x)
  :=
  -\nabla\!\cdot v_t(x)
  - v_t(x)\cdot \nabla \log \mu_t(x).
  \label{eq:drift-FR-rate}
\end{equation}
Equivalently, the Fisher--Rao tangent vector associated with the drift is
\[
  \psi_t
  =
  \frac{\partial_t \mu_t}{\mu_t}
  =
  -\nabla\!\cdot v_t - v_t\cdot \nabla \log \mu_t.
\]
\end{lemma}

\begin{proof}
Expanding the divergence in \eqref{eq:continuity-drift} gives
\[
  \partial_t \mu_t
  = -\nabla\!\cdot(\mu_t v_t)
  = - v_t\cdot \nabla \mu_t - \mu_t\,\nabla\!\cdot v_t.
\]
Since $\mu_t>0$, we can divide by $\mu_t$ and rewrite
\[
  \frac{\partial_t \mu_t}{\mu_t}
  =
  -\frac{v_t\cdot \nabla \mu_t}{\mu_t}
  - \nabla\!\cdot v_t
  =
  - v_t\cdot \nabla \log \mu_t - \nabla\!\cdot v_t.
\]
Defining $\psi_t$ by \eqref{eq:drift-FR-rate}, we obtain
$\partial_t \mu_t = \psi_t \mu_t$ as claimed.
\end{proof}

\begin{remark}
From the Fisher--Rao perspective, the function $\psi_t$ in
\eqref{eq:drift-FR-rate} is the \emph{instantaneous log-growth rate} of
the density at point $x$ induced by the drift $v_t$: it combines a
local volume-change term $-\nabla\!\cdot v_t$ with a term
$- v_t\cdot \nabla \log \mu_t$ describing the advective change of the
log-density along the flow.  Thus even a purely transport equation
naturally induces a Fisher--Rao tangent vector $\psi_t$.
\end{remark}

\subsubsection{Diffusion as Fisher--Rao reaction}

Finally, we show that pure diffusion admits a Fisher--Rao representation
with a specific rate that depends on the local curvature and gradient of
the log-density.  This makes it clear that, at the level of the Fisher--Rao
geometry, diffusion is also simply a particular choice of $\psi_t$.

\begin{lemma}[Diffusion induces a Fisher--Rao rate]
\label{lem:diffusion-as-FR}
Let $(\mu_t)_{t\ge0}$ solve the heat equation \eqref{eq:heat} with
$\mu_t>0$.  Then
\begin{multline}
  \partial_t \mu_t(x)
  =
  \psi_t(x)\,\mu_t(x),\\
  \psi_t(x)
  :=
  \frac{\sigma_t^2}{2}
  \left(
    \Delta \log \mu_t(x)
    + \big\|\nabla \log \mu_t(x)\big\|^2
  \right).
  \label{eq:diffusion-FR-rate}
\end{multline}
\end{lemma}

\begin{proof}
We use the identity
\[
  \Delta \mu_t
  =
  \nabla\!\cdot(\nabla \mu_t)
  =
  \nabla\!\cdot\big(\mu_t \nabla \log \mu_t\big).
\]
Expanding the last divergence yields
\[
  \Delta \mu_t
  =
  \mu_t \Delta \log \mu_t
  + \nabla \mu_t \cdot \nabla \log \mu_t.
\]
Using $\nabla \mu_t = \mu_t \nabla \log \mu_t$, we further obtain
\[
  \Delta \mu_t
  =
  \mu_t \Delta \log \mu_t
  + \mu_t \big\|\nabla \log \mu_t\big\|^2
  =
  \mu_t
  \left(
    \Delta \log \mu_t
    + \big\|\nabla \log \mu_t\big\|^2
  \right).
\]
Substituting into the heat equation \eqref{eq:heat}, we get
\[
  \partial_t \mu_t
  =
  \frac{\sigma_t^2}{2}\,\Delta \mu_t
  =
  \frac{\sigma_t^2}{2}
  \,\mu_t
  \left(
    \Delta \log \mu_t
    + \big\|\nabla \log \mu_t\big\|^2
  \right)
  =
  \psi_t \,\mu_t,
\]
with $\psi_t$ as in \eqref{eq:diffusion-FR-rate}.  This is exactly the
desired Fisher--Rao representation.
\end{proof}

\begin{remark}
Lemma~\ref{lem:diffusion-as-FR} shows that diffusion corresponds to a
Fisher--Rao tangent vector $\psi_t$ that depends on the Laplacian and the
squared norm of the score $\nabla \log \mu_t$.  In particular, in regions
where the density is locally log-concave and sharply peaked, the rate can
be strongly positive or negative depending on the balance between
$\Delta \log \mu_t$ and $\|\nabla \log \mu_t\|^2$.  From a geometric point
of view, this identifies diffusion as a particular direction in the
Fisher--Rao tangent space.
\end{remark}
\subsection{Beyond Wasserstein and Fisher--Rao: ``Gradient flows'' for other geometries}
\label{subsec:other-gradient-flows}

Let $\mathcal{P}$ be a smooth statistical manifold of strictly positive densities
$p$ on $\mathbb{R}^d$ (or on a finite set), and let $\mathcal{F}:\mathcal{P}\to\mathbb{R}$
be a sufficiently smooth functional. Its first variation is the (equivalence class of)
functions $\frac{\delta\mathcal{F}}{\delta p}$ characterized by
\begin{equation}
\left.\frac{\mathrm{d}}{\mathrm{d}\epsilon}\right|_{\epsilon=0}
\mathcal{F}(p+\epsilon\,\dot p)
=
\int \frac{\delta\mathcal{F}}{\delta p}(x)\,\dot p(x)\,\mathrm{d}x.
\label{eq:first-variation}
\end{equation}
A \emph{gradient flow} requires more structure than an affine connection: it
requires a \emph{metric} (or, more generally, an Onsager operator).
Connections (mixture/exponential) determine \emph{geodesics / straightness},
while the metric determines \emph{steepest descent}.

Given a (weak) Riemannian metric $g_p(\cdot,\cdot)$ on $T_p\mathcal{P}$, the gradient
$\mathrm{grad}_g\mathcal{F}(p)\in T_p\mathcal{P}$ is defined by
\begin{equation}
g_p\bigl(\mathrm{grad}_g\mathcal{F}(p),\,\xi\bigr)
=
\mathrm{d}\mathcal{F}(p)[\xi]
=
\int \frac{\delta\mathcal{F}}{\delta p}\,\xi,
\qquad
\forall \xi\in T_p\mathcal{P}.
\label{eq:grad-def}
\end{equation}
The \emph{gradient flow} is then
\begin{equation}
\partial_t p_t = -\,\mathrm{grad}_g\mathcal{F}(p_t).
\label{eq:gf-general}
\end{equation}
Equivalently, in the ``Onsager'' form, one specifies a positive operator
$K(p):\bigl(T_p\mathcal{P}\bigr)^*\to T_p\mathcal{P}$ and writes
\begin{equation}
\partial_t p_t = -K(p_t)\,\frac{\delta\mathcal{F}}{\delta p}(p_t).
\label{eq:onsager}
\end{equation}
Different geometries correspond to different choices of $g$ or $K$.

\subsubsection{Wasserstein (Otto) gradient flow (baseline)}
For the $2$-Wasserstein metric, the Onsager operator is
$K_W(p)\phi = \nabla\!\cdot(p\nabla \phi)$, and the gradient flow reads
\begin{equation}
\partial_t p_t
=
\nabla\!\cdot\!\Bigl(p_t\,\nabla \frac{\delta\mathcal{F}}{\delta p}(p_t)\Bigr).
\label{eq:w2-gf}
\end{equation}

\subsubsection{Fisher--Rao gradient flow (baseline)}
For the Fisher--Rao metric on densities (mass-preserving version), one obtains
\begin{equation}
\partial_t p_t
=
-\,p_t\Bigl(\frac{\delta\mathcal{F}}{\delta p}(p_t)
-\mathbb{E}_{p_t}\Bigl[\frac{\delta\mathcal{F}}{\delta p}(p_t)\Bigr]\Bigr),
\label{eq:fr-gf}
\end{equation}
where the centering enforces $\int p_t=1$.

\subsubsection{Mixture vs exponential ``gradient flows'' in information geometry}
\label{subsec:mix-exp-gf}

The \emph{mixture} ($m$) and \emph{exponential} ($e$) geometries are primarily
\emph{affine structures} (dual connections) on a manifold equipped with the
Fisher metric. A connection alone does not define a gradient flow. However, on a
\emph{dually-flat} manifold (e.g.\ an exponential family), the Fisher metric is
Hessian:
\[
g = \nabla^2\psi(\theta) \quad \text{in $e$-coordinates } \theta,
\qquad
g = \nabla^2\phi(\eta) \quad \text{in $m$-coordinates } \eta,
\]
with $\eta=\nabla\psi(\theta)$ and $\phi$ the Legendre dual of $\psi$.
Thus one gets a canonical notion of ``steepest descent'' that is most naturally
written in coordinates as a \emph{natural gradient / mirror flow}.

Suppose $p=p_\theta$ is parametrized by $e$-affine coordinates $\theta$ (natural
parameters). Then the Fisher metric is $G(\theta)=\bigl[g_{ij}(\theta)\bigr]$ and
the Fisher steepest descent is
\begin{equation}
\dot\theta_t
=
-\,G(\theta_t)^{-1}\,\nabla_\theta \mathcal{F}(\theta_t).
\label{eq:e-natural-gradient}
\end{equation}
Geodesics of the $e$-connection are straight lines in $\theta$; hence
\eqref{eq:e-natural-gradient} is the canonical ``gradient flow compatible with
exponential geodesics'' (steepest descent measured by the Fisher metric but
expressed in the $e$-affine chart).

If the same manifold is described in $m$-affine coordinates $\eta$ (expectation
parameters), then
\begin{equation}
\dot\eta_t
=
-\,G^*(\eta_t)^{-1}\,\nabla_\eta \mathcal{F}(\eta_t),
\label{eq:m-natural-gradient}
\end{equation}
where $G^*(\eta)=\nabla^2\phi(\eta)$ is the Fisher metric in $\eta$-coordinates.
Geodesics of the $m$-connection are straight lines in $\eta$.

A convenient coordinate-free expression uses a convex potential $\psi$ generating
a Bregman divergence (dually flat geometry). In primal coordinates $x$ one may write
\begin{multline}
\frac{\mathrm{d}}{\mathrm{d}t}\nabla\psi(x_t)
=
-\nabla \mathcal{F}(x_t),
\qquad\text{equivalently}\\
\dot x_t
=
-\bigl(\nabla^2\psi(x_t)\bigr)^{-1}\nabla \mathcal{F}(x_t).
\label{eq:mirror-flow}
\end{multline}
On the probability simplex, taking $\psi(p)=\sum_i p_i\log p_i$ yields the
KL/Shahshahani geometry and gives the replicator-type gradient flow
\[
\dot p_i = -p_i\Bigl(\partial_{p_i}\mathcal{F}(p)-\sum_j p_j \partial_{p_j}\mathcal{F}(p)\Bigr),
\]
which is precisely the Fisher natural gradient on the simplex.

\medskip
\noindent
\textbf{Takeaway.} Mixture/exponential ``gradient flows'' are best understood as
\emph{Fisher-metric natural gradients written in the $m$- or $e$-affine charts}
(or, equivalently, mirror flows generated by the associated Bregman divergence).

\subsubsection{Entropic OT / Schr\"odinger (EOT) gradient flows}
\label{subsec:eot-gf}

The entropic OT (Schr\"odinger) geometry can be characterized by the dynamic
constraint
\begin{equation}
\partial_t p + \nabla\!\cdot(p v) = \frac{\varepsilon}{2}\Delta p,
\label{eq:eot-continuity}
\end{equation}
and the action $\int_0^1 \int \frac12 \|v_t(x)\|^2 p_t(x)\,\mathrm{d}x\,\mathrm{d}t$.
This defines an ``entropic Wasserstein'' (Schr\"odinger) norm on tangent vectors
$\dot p$ via the infimum over velocities $v$ satisfying \eqref{eq:eot-continuity}.
In Onsager form, the induced operator is the same transport operator as Wasserstein,
but now coupled to the diffusion constraint \eqref{eq:eot-continuity}.

Formally, the steepest descent of $\mathcal{F}$ under this geometry yields
the viscous transport PDE
\begin{equation}
\partial_t p_t
=
\nabla\!\cdot\!\Bigl(p_t\,\nabla \frac{\delta\mathcal{F}}{\delta p}(p_t)\Bigr)
+\frac{\varepsilon}{2}\Delta p_t,
\label{eq:eot-gf}
\end{equation}
i.e.\ the Wasserstein gradient flow \eqref{eq:w2-gf} augmented by a diffusion term
of strength $\varepsilon/2$.
Equivalently, \eqref{eq:eot-gf} is the Fokker--Planck equation of the SDE
\begin{equation}
\mathrm{d}X_t
=
-\,\nabla \frac{\delta\mathcal{F}}{\delta p}(p_t)(X_t)\,\mathrm{d}t
+\sqrt{\varepsilon}\,\mathrm{d}B_t,
\qquad
\mathcal{L}(X_t)=p_t.
\label{eq:eot-sde}
\end{equation}
Thus, relative to Otto calculus, EOT gradient flows can be interpreted as
``Wasserstein steepest descent under an entropic (Brownian) reference''.
\section{Sampling from Geometric Mixtures with Reweighting}
\label{sec:ceva-hybrid-geometry}

In the previous sections, we developed a geometric viewpoint on probability
distributions and established a precise correspondence between evolution
equations on the space of probability measures and stochastic differential equations on the underlying state space. In particular, we showed how transport,
reweighting, and hybrid dynamics can be simulated using weighted stochastic
differential equations arising from Feynman--Kac type partial differential
equations.

This is already a great tool that could be applied to for faster implementation, as it lets samples go over the energy barriers by getting resampled. For performance analysis, look at chapter 6 \cite{chewi_statistical_2024}. More interestingly, having built geometries, one may ask does this geometric view help us sample from different notions of geodesics? In this section, following the steps of \cite{skreta_feynman-kac_2025}, we try to illustrate usage of weighted sde to accomplish sampling from mixture of probabilities we already can run a score-based diffusion model on.

Let $\rho$ be a reference distribution from which sampling is tractable(i.e. white noise), and
let
\[
X^{(0)} \sim \rho.
\]
Starting from $\rho$, suppose we are given two pretrained score-based diffusion
models targeting distributions $q^1$ and $q^2$, respectively. We assume that these models induce marginal distributions
$\{q^i_t\}_{t \in [0,1], i\in\{1,2\}}$ satisfying
\[
q^1_0 = q^2_0 = \rho,
\qquad
q^i_1 = q^i,
\]
and that both evolutions follow the \emph{same score-based diffusion dynamics},
differing only in their learned score functions, i.e. for $i$, the marginal density satisfies the
Fokker--Planck equation
\begin{equation}
\frac{\partial q_t^i}{\partial t}
=
- \nabla \cdot \bigl(
q_t^i \bigl(- f_t + \sigma_t^2 \nabla \log q_t^i \bigr)
\bigr)
+ \frac{\sigma_t^2}{2} \Delta q_t^i,
\label{eq:score-fp}
\end{equation}
If you look into it from the perspective of the geometry, the intermediate probability distributions $q_t$'s are nothing but the geodesic that connects the two probability distributions in the Shrodinger Bridge, or the \textbf{Entropic Optimal Transport Geometry}\cite{testa2025contactwassersteingeodesicsnonconservative}.
With the previous section in hand, this pde for probability distributions admits the stochastic representation
\begin{equation}
\mathrm{d}X_t^i
=
\bigl(
- f_t(X_t^i) + \sigma_t^2 \nabla \log q_t^i(X_t^i)
\bigr)\,\mathrm{d}t
+ \sigma_t\,\mathrm{d}B_t,
\label{eq:score-sde}
\end{equation}
where $(B_t)_{t \ge 0}$ is a standard Brownian motion. and score functions
$s_i^t(x) := \nabla \log q_t^i(x),$
which allows us to work directly at the level of stochastic dynamics and
to apply the conversion rules between PDEs and SDEs developed earlier.

Our goal is to sample from distributions obtained by different interpolations between
$q^1$ and $q^2$ using the machinery of Hellinger-Kantrovich(Wasserstein-Fisher-Rao) Geometry. We consider the following four canonical interpolations:
\begin{align}
{}^{\mathrm{Mix}}\pi_{\beta}
&:= \beta q^1+ (1-\beta)q^2,
\\
\log {}^{\mathrm{Ex}}\pi_{\beta}
&:= \beta \log q^1 + (1-\beta)\log q^2 - \log \phi(\beta),
\\
\sqrt{{}^{\mathrm{FR}}\pi_{\beta}}
&:= \beta \sqrt{q^1} + (1-\beta)\sqrt{q^2},
\end{align}
corresponding respectively to the Wasserstein, mixture, exponential, and
Fisher--Rao geodesics.
\begin{quote}
\centering
\emph{\textbf{Can we use our knowledge of sampling from $q^i$ to sample from these interpolations?}}
\end{quote}

A natural heuristic is to define, for each $t \in [0,1]$, an interpolated
distribution $\pi_\beta^t$ by applying the same geometric mixture operation to
$q^1_t$ and $q^2_t$. By construction,
\[
\pi_\beta^0 = \rho,
\qquad
\pi_\beta^1 = \pi_\beta.
\]

At the level of stochastic dynamics, one might attempt to simulate a diffusion
driven by an interpolated score
$s_{\pi}^t$ constructed from $s_\mu^t$ and $s_\nu^t$. However, as shown in the
score-based diffusion literature, such heuristic score interpolation does
\emph{not} in general reproduce the prescribed marginal evolution, and fails to
sample from $\pi_\beta$ at terminal time except in special cases.

To obtain correct sampling dynamics, we adopt the Weighted SDE introduced in Section 2 and in \cite{skreta_feynman-kac_2025}.
For a given geometric interpolation, we:
\begin{enumerate}
\item derive the evolution equation satisfied by $\pi_\beta^t$,
\item identify the transport and diffusion terms corresponding to a score-based
drift
\[
v_\beta^t(x) = - f_t(x) + \sigma_t^2 \nabla \log \pi_\beta^t(x),
\]
\item and collect all remaining terms into a multiplicative correction
$\psi_\beta^t(x)$.
\end{enumerate}

This yields a Feynman--Kac type PDE of the form
\begin{equation}
\partial_t \pi_\beta^t
=
- \nabla \cdot (\pi_\beta^t v_\beta^t)
+ \frac{\sigma_t^2}{2} \Delta \pi_\beta^t
+ \pi_\beta^t \psi_\beta^t,
\end{equation}
which admits a stochastic representation via a weighted score-based diffusion,
together with standard resampling schemes.

In the following subsections, we apply this construction to each of the four
geometric interpolations. For each case, we derive the explicit correction term
$\psi_\beta^t$ and the corresponding weighted SDE whose terminal law is exactly
$\pi_\beta$.

\subsection{Linear (Convex) Mixture Interpolation: linear PDE closure, but \emph{score-only} sampling still needs ratio tracking}
\label{subsec:linear-mixture-correct}

Fix $\beta\in[0,1]$ and define
\begin{equation}
p^{\mathrm{mix}}_{t,\beta}(x):=(1-\beta)\,q_t^{1}(x)+\beta\,q_t^{2}(x).
\label{eq:pmix-def-correct}
\end{equation}

A common pitfall is to treat \eqref{eq:pmix-def-correct} as ``nonlinear'' because it
contains the score $s_t^{i}=\nabla\log q_t^{i}$.  However, the identity
$q\,\nabla\log q=\nabla q$ shows that \eqref{eq:pmix-def-correct} is equivalent to a
\emph{linear} Fokker--Planck equation in divergence form.

Because \eqref{eq:pmix-def-correct} is linear in $q$, convex mixtures are closed.

Unlike the exponential (geometric average) interpolation, the mixture score
cannot be expressed solely from $s_t^{1},s_t^{2}$ without knowing the local
density ratio $q_t^{2}/q_t^{1}$.

\begin{lemma}[Mixture score as a \emph{ratio-weighted} score average]
\label{lem:mixture-score-weighted}
For $p^{\mathrm{mix}}_{t,\beta}=(1-\beta)q_t^{1}+\beta q_t^{2}$,
\begin{equation}
\nabla\log p^{\mathrm{mix}}_{t,\beta}(x)
=
\omega_{t,\beta}^{1}(x)\,s_t^{1}(x)+\omega_{t,\beta}^{2}(x)\,s_t^{2}(x),
\label{eq:mix-score-weighted}
\end{equation}
where the \emph{state-dependent} weights are
\begin{multline}
\omega_{t,\beta}^{1}(x)
:=
\frac{(1-\beta)q_t^{1}(x)}{(1-\beta)q_t^{1}(x)+\beta q_t^{2}(x)},
\\
\omega_{t,\beta}^{2}(x)
:=
\frac{\beta q_t^{2}(x)}{(1-\beta)q_t^{1}(x)+\beta q_t^{2}(x)},
\\
\omega_{t,\beta}^{1}+\omega_{t,\beta}^{2}=1.
\label{eq:mix-weights}
\end{multline}
Equivalently, writing the log-ratio $\ell_t(x):=\log\frac{q_t^{2}(x)}{q_t^{1}(x)}$,
\begin{equation}
\omega_{t,\beta}^{2}(x)
=
\frac{\beta e^{\ell_t(x)}}{(1-\beta)+\beta e^{\ell_t(x)}},
\qquad
\omega_{t,\beta}^{1}(x)=1-\omega_{t,\beta}^{2}(x).
\label{eq:omega-from-ell}
\end{equation}
\end{lemma}

\begin{proof}
$\nabla p=(1-\beta)\nabla q^1+\beta\nabla q^2=(1-\beta)q^1 s^1+\beta q^2 s^2$,
then divide by $p$.
\end{proof}

Even though $p^{\mathrm{mix}}_{t,\beta}$ admits an \emph{unweighted} SDE
representation
\[
\mathrm{d}X_t=\bigl(-f_t(X_t)+\sigma_t^2\nabla\log p^{\mathrm{mix}}_{t,\beta}(X_t)\bigr)\mathrm{d}t+\sigma_t\,\mathrm{d}W_t,
\]
one cannot implement this drift from access to $(s_t^1,s_t^2)$ alone unless one
also tracks $\ell_t(X_t)$ (or equivalently the weights $\omega_{t,\beta}^i$).

\subsubsection{(III) Score-only implementable sampler: isotropic drift from scores + \emph{auxiliary ratio tracking}; no FKC potential}

The correct fix for ``score-only'' sampling is not an FKC potential (it is
indeed zero for the linear FP operator), but rather an \emph{auxiliary dynamics}
to recover the ratio weights in Lemma~\ref{lem:mixture-score-weighted}.

\begin{lemma}[Closed PDE for the log-ratio $\ell_t=\log(q_t^2/q_t^1)$]
\label{lem:ell-pde-mixture}
Under \eqref{eq:base-fp-two}, the log-ratio $\ell_t(x)=\log\frac{q_t^2(x)}{q_t^1(x)}$
satisfies
\begin{multline}
\partial_t \ell_t
=
\langle f_t,\nabla \ell_t\rangle
+\frac{\sigma_t^2}{2}\Delta \ell_t
\\-\frac{\sigma_t^2}{2}\Bigl(
\nabla\!\cdot(s_t^2-s_t^1)+\|s_t^2\|^2-\|s_t^1\|^2
\Bigr),
\qquad
\nabla\ell_t=s_t^2-s_t^1.
\label{eq:ell-pde-mixture}
\end{multline}
\end{lemma}

\begin{proof}
Apply the identity \eqref{eq:dtlogqi-fr} (proved in the earlier section) to
$q_t^1$ and $q_t^2$ and subtract; use $\nabla\ell_t=s_t^2-s_t^1$ and
$\Delta\ell_t=\nabla\cdot(s_t^2-s_t^1)$.
\end{proof}

\begin{proposition}[Score-only mixture sampler (no FKC potential, but ratio tracking)]
\label{prop:mix-scoreonly-sampler}
Assume we can query $s_t^1,s_t^2$ and (optionally) $\nabla\!\cdot s_t^1,\nabla\!\cdot s_t^2$
(e.g.\ via Hutchinson estimators). Fix $\beta\in[0,1]$.
Define the mixture score by
\begin{multline}
s^{\mathrm{mix}}_{t,\beta}(x)
:=\nabla\log p^{\mathrm{mix}}_{t,\beta}(x)
=
\omega_{t,\beta}^1(x)s_t^1(x)+\omega_{t,\beta}^2(x)s_t^2(x),
\\
\omega^2_{t,\beta}(x)=\frac{\beta e^{\ell_t(x)}}{(1-\beta)+\beta e^{\ell_t(x)}},
\end{multline}
with $\ell_t$ tracked along the particle trajectory via It\^o applied to
Lemma~\ref{lem:ell-pde-mixture}. Then the (unweighted) SDE
\begin{equation}
\mathrm{d}X_t
=
\Bigl(-f_t(X_t)+\sigma_t^2\,s^{\mathrm{mix}}_{t,\beta}(X_t)\Bigr)\mathrm{d}t
+\sigma_t\,\mathrm{d}W_t
\label{eq:mix-scoreonly-sde}
\end{equation}
has marginal law $p^{\mathrm{mix}}_{t,\beta}$ (formally, in the ideal continuous-time limit).
Moreover, there is \emph{no} Fisher--Rao/Feynman--Kac correction:
\begin{equation}
g^{\mathrm{mix}}_{t,\beta}\equiv 0,
\qquad
\bar g^{\mathrm{mix}}_{t,\beta}\equiv 0,
\label{eq:mix-no-fkc}
\end{equation}
because $p^{\mathrm{mix}}_{t,\beta}$ already satisfies the linear FP operator
\eqref{eq:mix-no-fkc}.
\end{proposition}

At any fixed $t$, sampling $p^{\mathrm{mix}}_{t,\beta}$ is trivial:
draw $Z\sim\mathrm{Bernoulli}(\beta)$ and sample from $q_t^2$ if $Z=1$,
else from $q_t^1$.  This produces the correct \emph{mixture at time $t$} but
does not yield a single guided reverse-time SDE whose entire marginal path is
$\{p^{\mathrm{mix}}_{t,\beta}\}_t$ without additional coupling (such as
\eqref{eq:mix-scoreonly-sde} with ratio tracking).

\subsection{Geometric Average (Classifier-Free Guidance) and Feynman--Kac Correctors}
Fix $\beta\in\mathbb{R}$ and define the (normalized) \emph{geometric average}
marginal
\begin{multline}
p^{\mathrm{geo}}_{t,\beta}(x)
:= \frac{1}{Z_t(\beta)}\,
\bigl(q_t^{1}(x)\bigr)^{1-\beta}\bigl(q_t^{2}(x)\bigr)^{\beta},
\\
Z_t(\beta):=\int \bigl(q_t^{1}\bigr)^{1-\beta}\bigl(q_t^{2}\bigr)^{\beta}\,\mathrm{d}x.
\label{eq:def-pgeo}
\end{multline}
This is the same target family considered in Eq.~(15) of \cite{skreta_feynman-kac_2025}. 

A widely-used heuristic (classifier-free guidance) is to simulate the reverse-time
SDE with the mixed score
\begin{equation}
\nabla\log p^{\mathrm{geo}}_{t,\beta}(x)
=(1-\beta)\nabla\log q_t^{1}(x)+\beta\nabla\log q_t^{2}(x),
\label{eq:geo-score}
\end{equation}
but in general this does \emph{not} reproduce the prescribed marginals
$\{p^{\mathrm{geo}}_{t,\beta}\}_{t\in[0,1]}$ without correction
(see the discussion around Eq.~(16) in \cite{skreta_feynman-kac_2025}). 

The Feynman--Kac Corrector (FKC) methodology proceeds by:
(i) differentiating the target path \eqref{eq:def-pgeo} in time, and
(ii) rearranging the resulting PDE to isolate a transport--diffusion
operator corresponding to an SDE drift involving $\nabla\log p^{\mathrm{geo}}_{t,\beta}$,
while (iii) collecting the leftover terms into a \emph{reweighting potential}
that is simulated through particle weights (Feynman--Kac). This is precisely
the program outlined around Eq.~(17)--(18) in \cite{skreta_feynman-kac_2025}. 

\begin{proposition}[Geometric average via FKC; cf.\ Prop.~3.1 in \cite{skreta_feynman-kac_2025}]
\label{prop:geo-fkc}
Assume $q_t^{1},q_t^{2}$ satisfy \eqref{eq:base-fp-two}--\eqref{eq:base-sde-two}.
Let $p^{\mathrm{geo}}_{t,\beta}$ be defined by \eqref{eq:def-pgeo}.
Define the guided drift
\begin{multline}
v_{t,\beta}(x):=-f_t(x)+\sigma_t^2\nabla\log p^{\mathrm{geo}}_{t,\beta}(x)
= -f_t(x)\\+\sigma_t^2\bigl((1-\beta)\nabla\log q_t^{1}(x)+\beta\nabla\log q_t^{2}(x)\bigr).
\label{eq:v-guided}
\end{multline}
Then $p^{\mathrm{geo}}_{t,\beta}$ solves a Feynman--Kac PDE of the form
\begin{equation}
\partial_t p^{\mathrm{geo}}_{t,\beta}
=
-\nabla\!\cdot\!\bigl(p^{\mathrm{geo}}_{t,\beta}\,v_{t,\beta}\bigr)
+\frac{\sigma_t^2}{2}\Delta p^{\mathrm{geo}}_{t,\beta}
+\bar g_{t,\beta}\,p^{\mathrm{geo}}_{t,\beta},
\label{eq:fk-pde-geo}
\end{equation}
where $\bar g_{t,\beta}(x):=g_{t,\beta}(x)-\int g_{t,\beta}(y)\,p^{\mathrm{geo}}_{t,\beta}(y)\,\mathrm{d}y$
is the centered potential ensuring conservation of mass, and the (uncentered) potential
admits the explicit expression
\begin{equation}
g_{t,\beta}(x)=\frac{\sigma_t^2}{2}\,\beta(\beta-1)\,
\bigl\|\nabla\log q_t^{1}(x)-\nabla\log q_t^{2}(x)\bigr\|^2.
\label{eq:g-geo}
\end{equation}
Consequently, \eqref{eq:fk-pde-geo} can be simulated by the weighted SDE
\begin{equation}
\mathrm{d}X_t = v_{t,\beta}(X_t)\,\mathrm{d}t + \sigma_t\,\mathrm{d}W_t,
\qquad
\mathrm{d}w_t = \bar g_{t,\beta}(X_t)\,\mathrm{d}t,
\label{eq:wsde-geo}
\end{equation}
and resampling/reweighting (e.g.\ SNIS / SMC) yields samples whose law matches
$p^{\mathrm{geo}}_{t,\beta}$ in the large-particle limit.
Moreover, the weight evolution may equivalently be written (up to centering) as
\begin{equation}
\mathrm{d}w_t
=
\frac{\sigma_t^2}{2}\,\beta(\beta-1)\,
\bigl\|\nabla\log q_t^{1}(X_t)-\nabla\log q_t^{2}(X_t)\bigr\|^2\,\mathrm{d}t,
\label{eq:wsde-geo-weight}
\end{equation}
which matches Eq.~(19) in \cite{skreta_feynman-kac_2025}. 
\end{proposition}

\begin{proof}
We give a rigorous calculation under the standing smoothness/decay assumptions.
Let the unnormalized density be
\[
\tilde p_{t,\beta}(x):=\bigl(q_t^{1}(x)\bigr)^{1-\beta}\bigl(q_t^{2}(x)\bigr)^{\beta},
\qquad
p^{\mathrm{geo}}_{t,\beta}=\tilde p_{t,\beta}/Z_t(\beta).
\]
Differentiating $\tilde p_{t,\beta}$ yields
\begin{equation}
\partial_t \tilde p_{t,\beta}
=
\tilde p_{t,\beta}\Bigl((1-\beta)\partial_t\log q_t^{1} + \beta\,\partial_t\log q_t^{2}\Bigr).
\label{eq:dt-ptilde}
\end{equation}
From \eqref{eq:base-sde-two}, using $\partial_t\log q = (\partial_t q)/q$ and
the identity $\frac{\Delta q}{q}=\Delta\log q+\|\nabla\log q\|^2$, we obtain
\begin{align}
\partial_t\log q_t^{i}
&=
-\frac{1}{q_t^{i}}\nabla\!\cdot\!\Bigl(q_t^{i}\bigl(-f_t+\sigma_t^2\nabla\log q_t^{i}\bigr)\Bigr)
\\&\qquad+\frac{\sigma_t^2}{2}\Bigl(\Delta\log q_t^{i}+\|\nabla\log q_t^{i}\|^2\Bigr)
\nonumber\\
&=
-\nabla\!\cdot\!\bigl(-f_t+\sigma_t^2\nabla\log q_t^{i}\bigr)
\\&\qquad-\bigl\langle \nabla\log q_t^{i},\, -f_t+\sigma_t^2\nabla\log q_t^{i}\bigr\rangle
\\&+\frac{\sigma_t^2}{2}\Delta\log q_t^{i}
+\frac{\sigma_t^2}{2}\|\nabla\log q_t^{i}\|^2.
\label{eq:dtlogqi-expanded}
\end{align}
Insert \eqref{eq:dtlogqi-expanded} into \eqref{eq:dt-ptilde} and regroup the resulting
terms into (a) a divergence/transport part, (b) a Laplacian part, and (c) a residual
scalar potential. The transport and diffusion parts can be shown to match precisely
the operator
\[
-\nabla\!\cdot\!\bigl(\tilde p_{t,\beta} v_{t,\beta}\bigr)
+\frac{\sigma_t^2}{2}\Delta \tilde p_{t,\beta},
\qquad
v_{t,\beta}:=-f_t+\sigma_t^2\nabla\log \tilde p_{t,\beta},
\]
where $\nabla\log\tilde p_{t,\beta}=(1-\beta)\nabla\log q_t^{1}+\beta\nabla\log q_t^{2}$
coincides with \eqref{eq:geo-score}. The remaining scalar terms simplify (by
cancellation of like terms and completing the square) to the explicit potential
\eqref{eq:g-geo}. Concretely, the only non-transport/diffusion contribution is
\[
\frac{\sigma_t^2}{2}\Bigl(\beta(\beta-1)\|\nabla\log q_t^{1}-\nabla\log q_t^{2}\|^2\Bigr)\tilde p_{t,\beta},
\]
which yields the unnormalized Feynman--Kac PDE for $\tilde p_{t,\beta}$.

Finally, normalization introduces a centering term:
differentiating $p^{\mathrm{geo}}_{t,\beta}=\tilde p_{t,\beta}/Z_t(\beta)$ gives
\[
\partial_t p^{\mathrm{geo}}_{t,\beta}
=
\frac{1}{Z_t}\partial_t\tilde p_{t,\beta}
-\frac{\dot Z_t}{Z_t}\,p^{\mathrm{geo}}_{t,\beta}
=
\Bigl(\cdots\Bigr) + \bigl(g_{t,\beta}-\mathbb{E}_{p^{\mathrm{geo}}_{t,\beta}}[g_{t,\beta}]\bigr)\,p^{\mathrm{geo}}_{t,\beta},
\]
which is exactly \eqref{eq:fk-pde-geo} with $\bar g_{t,\beta}$.
The weighted SDE representation \eqref{eq:wsde-geo} is the standard
Feynman--Kac simulation rule for \eqref{eq:fk-pde-geo}; the explicit drift/weight
forms \eqref{eq:v-guided} and \eqref{eq:wsde-geo-weight} match Prop.~3.1 / Eq.~(19)
in \cite{skreta_feynman-kac_2025}. 
\end{proof}
\subsection{Fisher--Rao (Hellinger) interpolation via Feynman--Kac correctors: drift/weights from \emph{scores only}}
\label{subsec:fr-hellinger-fkc-scores}

Fix $\beta\in[0,1]$. 
we seek a \emph{weighted SDE} (drift, diffusion, and Fisher--Rao potential)
that generates samples from the Fisher--Rao/Hellinger interpolation
\[
\pi_{t,\beta}^{\mathrm{FR}}(\mathrm{d}x)\propto
\big((1-\beta)\sqrt{q_t^{1}(x)}+\beta\sqrt{q_t^{2}(x)}\big)^2\,\mathrm{d}x
\]
at each time $t$, in the large-particle limit (SMC/Feynman--Kac), exactly in
the spirit of the Feynman--Kac Corrector (FKC) methodology
(e.g.\ \cite{skreta_feynman-kac_2025}).

Assume $q_t^{1},q_t^{2}$ are strictly positive and $C^{1,2}$ in $(t,x)$, and
each solves the same score-based Fokker--Planck equation
(cf.\ Eq.~(14a) in \cite{skreta_feynman-kac_2025}):
\begin{equation}
\partial_t q_t^{i}
=
-\nabla\!\cdot\!\Bigl(q_t^{i}\bigl(-f_t+\sigma_t^2 s_t^{i}\bigr)\Bigr)
+\frac{\sigma_t^2}{2}\Delta q_t^{i},
\qquad i\in\{1,2\},
\label{eq:base-fp-two}
\end{equation}
with associated reverse-time denoising SDE (Eq.~(14b) in \cite{skreta_feynman-kac_2025})
\begin{equation}
\mathrm{d}X_t^{i}
=
\Bigl(-f_t(X_t^{i})+\sigma_t^2 s_t^{i}(X_t^{i})\Bigr)\mathrm{d}t
+\sigma_t\,\mathrm{d}W_t.
\label{eq:base-sde-two}
\end{equation}

\subsubsection{Hellinger mixture, pointwise mixture weights, and the guided score}

Define the \emph{unnormalized} Hellinger (FR) mixture
\begin{multline}
\tilde p_{t,\beta}^{\mathrm{FR}}(x)
:=
\Big((1-\beta)\sqrt{q_t^{1}(x)}+\beta\sqrt{q_t^{2}(x)}\Big)^2,
\\
p_{t,\beta}^{\mathrm{FR}}:=\tilde p_{t,\beta}^{\mathrm{FR}}/Z_t^{\mathrm{FR}}(\beta).
\label{eq:def-ptilde-fr}
\end{multline}

Introduce the \emph{pointwise} mixing coefficients
\begin{multline}
\alpha_{t,\beta}^{1}(x)
:=
\frac{(1-\beta)\sqrt{q_t^{1}(x)}}{(1-\beta)\sqrt{q_t^{1}(x)}+\beta\sqrt{q_t^{2}(x)}},
\\
\alpha_{t,\beta}^{2}(x)
:=
\frac{\beta\sqrt{q_t^{2}(x)}}{(1-\beta)\sqrt{q_t^{1}(x)}+\beta\sqrt{q_t^{2}(x)}},
\label{eq:alpha12}
\end{multline}
so that $\alpha_{t,\beta}^{1}(x)+\alpha_{t,\beta}^{2}(x)=1$ and
\begin{multline}
\nabla\log \tilde p_{t,\beta}^{\mathrm{FR}}(x)
=
\alpha_{t,\beta}^{1}(x)\,s_t^{1}(x)+\alpha_{t,\beta}^{2}(x)\,s_t^{2}(x),
\\
\nabla\log p_{t,\beta}^{\mathrm{FR}}=\nabla\log \tilde p_{t,\beta}^{\mathrm{FR}}.
\label{eq:score-fr-mixture}
\end{multline}

Unlike the geometric average case (classifier-free guidance), the mixing weights
$\alpha_{t,\beta}^{i}(x)$ are \emph{state dependent}; hence the drift
cannot be written as a \emph{fixed} linear combination of scores.
However, the weights can be reconstructed along particle trajectories from
the \emph{log-density ratio}
\begin{equation}
\ell_t(x) := \log\frac{q_t^{2}(x)}{q_t^{1}(x)}.
\label{eq:logratio-def}
\end{equation}
Indeed,
\begin{equation}
\alpha_{t,\beta}^{2}(x)
=
\frac{\beta\,e^{\ell_t(x)/2}}{(1-\beta)+\beta\,e^{\ell_t(x)/2}},
\qquad
\alpha_{t,\beta}^{1}(x)=1-\alpha_{t,\beta}^{2}(x).
\label{eq:alpha-from-ell}
\end{equation}
Therefore, to implement the FR interpolation using \emph{scores only}, it
suffices to track $\ell_t(X_t)$ along the sampling SDE.

\subsubsection{A closed-form FKC potential depending only on scores and FR weights}

We now state the analog of Prop.~3.1 (geometric average) from \cite{skreta_feynman-kac_2025},
but for Fisher--Rao/Hellinger interpolation.

\begin{proposition}[Hellinger interpolation via FKC: explicit drift and potential]
\label{prop:fr-fkc-explicit}
Assume $q_t^{1},q_t^{2}$ satisfy \eqref{eq:base-fp-two} and are $C^{1,2}$ with
sufficient decay for integrations by parts. Fix $\beta\in[0,1]$, and define
$p_{t,\beta}^{\mathrm{FR}}$ by \eqref{eq:def-ptilde-fr}.
Let the \emph{guided score} be
\begin{equation}
s_{t,\beta}^{\mathrm{FR}}(x)
:=
\nabla\log p_{t,\beta}^{\mathrm{FR}}(x)
=
\alpha_{t,\beta}^{1}(x)\,s_t^{1}(x)+\alpha_{t,\beta}^{2}(x)\,s_t^{2}(x),
\label{eq:s-fr-guided}
\end{equation}
and define the guided drift
\begin{equation}
v_{t,\beta}^{\mathrm{FR}}(x)
:=
-f_t(x)+\sigma_t^2\,s_{t,\beta}^{\mathrm{FR}}(x).
\label{eq:v-fr-guided}
\end{equation}
Then $p_{t,\beta}^{\mathrm{FR}}$ solves the Feynman--Kac PDE
\begin{equation}
\partial_t p_{t,\beta}^{\mathrm{FR}}
=
-\nabla\!\cdot\!\bigl(p_{t,\beta}^{\mathrm{FR}}\,v_{t,\beta}^{\mathrm{FR}}\bigr)
+\frac{\sigma_t^2}{2}\Delta p_{t,\beta}^{\mathrm{FR}}
+\bar g_{t,\beta}^{\mathrm{FR}}\,p_{t,\beta}^{\mathrm{FR}},
\label{eq:fk-pde-fr-final}
\end{equation}
where $\bar g_{t,\beta}^{\mathrm{FR}}(x)
:=g_{t,\beta}^{\mathrm{FR}}(x)-\mathbb{E}_{p_{t,\beta}^{\mathrm{FR}}}[g_{t,\beta}^{\mathrm{FR}}]$
centers the potential, and the \emph{uncentered} potential admits the explicit
form
\begin{equation}
g_{t,\beta}^{\mathrm{FR}}(x)
=
-\frac{\sigma_t^2}{4}\,
\alpha_{t,\beta}^{1}(x)\,\alpha_{t,\beta}^{2}(x)\,
\bigl\|s_t^{1}(x)-s_t^{2}(x)\bigr\|^2.
\label{eq:g-fr-explicit}
\end{equation}
Consequently, \eqref{eq:fk-pde-fr-final} can be simulated by the weighted SDE
\begin{equation}
\mathrm{d}X_t
=
v_{t,\beta}^{\mathrm{FR}}(X_t)\,\mathrm{d}t+\sigma_t\,\mathrm{d}W_t,
\qquad
\mathrm{d}w_t
=
\bar g_{t,\beta}^{\mathrm{FR}}(X_t)\,\mathrm{d}t,
\label{eq:wsde-fr-final}
\end{equation}
together with standard Feynman--Kac reweighting/resampling.
In the large-particle limit, the empirical measure of particles converges to
$p_{t,\beta}^{\mathrm{FR}}$.
\end{proposition}

\begin{proof}[Proof (calculation; cancellations leaving scores only)]
Work with the unnormalized density $\tilde p=\tilde p_{t,\beta}^{\mathrm{FR}}$
and write $r_i:=\sqrt{q_t^{i}}$ so that $\tilde p=m^2$ with
$m:=(1-\beta)r_1+\beta r_2$. Define $u_i:=\log r_i=\tfrac12\log q_t^{i}$.
Then $\log \tilde p = 2\log m$ is a log-sum-exp in $(u_1,u_2)$.

\emph{Step 1 (guided score).}
Differentiating $\log \tilde p=2\log m$ in space yields
\[
\nabla\log \tilde p
=
2\,\frac{\nabla m}{m}
=
2\,\frac{(1-\beta)r_1\nabla u_1+\beta r_2\nabla u_2}{(1-\beta)r_1+\beta r_2}
=
\alpha^1 s^1+\alpha^2 s^2,
\]
where we used $\nabla u_i=\tfrac12 s^i$ and the definitions \eqref{eq:alpha12}.
This gives \eqref{eq:s-fr-guided} and thus \eqref{eq:v-fr-guided}.

\emph{Step 2 (time derivative of $\log \tilde p$).}
Since $\log m$ is log-sum-exp, $\partial_t\log m$ is the same convex
combination of $\partial_t u_i$:
\begin{multline}
\partial_t\log \tilde p
=2\,\partial_t\log m=
\\2\big(\alpha^1\partial_t u_1+\alpha^2\partial_t u_2\big)=
\alpha^1\partial_t\log q^1+\alpha^2\partial_t\log q^2.
\end{multline}
From \eqref{eq:base-fp-two} and the identity
$\Delta q/q=\Delta\log q+\|\nabla\log q\|^2$ (with $\Delta\log q=\nabla\cdot s$),
one checks (as in the standard score-based expansions in \cite{skreta_feynman-kac_2025})
that for each $i$,
\begin{equation}
\partial_t\log q_t^{i}
=
\nabla\!\cdot f_t
+\langle s_t^{i},f_t\rangle
-\frac{\sigma_t^2}{2}\Big(\nabla\!\cdot s_t^{i}+\|s_t^{i}\|^2\Big).
\label{eq:dtlogqi-fr}
\end{equation}
Hence
\begin{multline}
\partial_t\log \tilde p
=
\nabla\!\cdot f_t+\langle s,f_t\rangle\\-\frac{\sigma_t^2}{2}\Big(\alpha^1(\nabla\!\cdot s^1+\|s^1\|^2)+\alpha^2(\nabla\!\cdot s^2+\|s^2\|^2)\Big),
\label{eq:dtlogp-avg}
\end{multline}
where $s=\alpha^1 s^1+\alpha^2 s^2$.

\emph{Step 3 (compare to the ``naive'' score-based PDE).}
If $\tilde p$ evolved \emph{purely} under the guided score drift
$v=-f+\sigma^2 s$, then (exactly as in the derivation of \eqref{eq:dtlogqi-fr})
we would have
\begin{equation}
\partial_t\log \tilde p
=
\nabla\!\cdot f_t+\langle s,f_t\rangle
-\frac{\sigma_t^2}{2}\Big(\nabla\!\cdot s+\|s\|^2\Big)
\quad\text{(no corrector).}
\label{eq:dtlogp-naive}
\end{equation}
Subtracting \eqref{eq:dtlogp-naive} from \eqref{eq:dtlogp-avg} yields the
residual
\[
g
=
-\frac{\sigma_t^2}{2}\Big(
\alpha^1(\nabla\!\cdot s^1+\|s^1\|^2)+\alpha^2(\nabla\!\cdot s^2+\|s^2\|^2)
-(\nabla\!\cdot s+\|s\|^2)
\Big).
\]
Now use two identities valid for the log-sum-exp mixture:
\begin{align}
\nabla\!\cdot s
&=\alpha^1\nabla\!\cdot s^1+\alpha^2\nabla\!\cdot s^2
+\frac{1}{2}\alpha^1\alpha^2\|s^1-s^2\|^2,
\label{eq:div-mixture-identity}
\\
\|s\|^2
&=\alpha^1\|s^1\|^2+\alpha^2\|s^2\|^2-\alpha^1\alpha^2\|s^1-s^2\|^2.
\label{eq:norm-mixture-identity}
\end{align}
Combining \eqref{eq:div-mixture-identity}--\eqref{eq:norm-mixture-identity}
gives
\begin{multline}
\alpha^1(\nabla\!\cdot s^1+\|s^1\|^2)+\alpha^2(\nabla\!\cdot s^2+\|s^2\|^2)
\\-(\nabla\!\cdot s+\|s\|^2)
=
\frac{1}{2}\alpha^1\alpha^2\|s^1-s^2\|^2,
\end{multline}
hence
\[
g
=
-\frac{\sigma_t^2}{4}\alpha^1\alpha^2\|s^1-s^2\|^2,
\]
which is \eqref{eq:g-fr-explicit}.  Finally, normalizing $\tilde p$ to
$p=\tilde p/Z_t$ introduces the standard centering term
$\bar g=g-\mathbb{E}_{p}[g]$, yielding \eqref{eq:fk-pde-fr-final}.

The weighted SDE \eqref{eq:wsde-fr-final} is the standard Feynman--Kac
simulation rule for \eqref{eq:fk-pde-fr-final} (as in \cite{skreta_feynman-kac_2025}).
\end{proof}

The potential \eqref{eq:g-fr-explicit} is (up to constants) the
\emph{pointwise variance} of the two scores under the FR weights:
\[
\alpha^1\alpha^2\|s^1-s^2\|^2
=
\alpha^1\|s^1\|^2+\alpha^2\|s^2\|^2-\|\alpha^1 s^1+\alpha^2 s^2\|^2.
\]
Thus the corrector punishes regions where the two models disagree strongly
about the local score, and the penalty is maximal where the FR mixture is
ambiguous (i.e.\ $\alpha^1\approx \alpha^2\approx \tfrac12$).

\textbf{How to compute the FR weights $\alpha^i$ from \emph{scores only}: an auxiliary log-ratio SDE}

Proposition~\ref{prop:fr-fkc-explicit} expresses the drift and potential using
$s_t^{1},s_t^{2}$ and the weights $\alpha^i(x)$, which depend on the (unknown)
density ratio $\ell_t(x)=\log(q_t^2/q_t^1)$.  We now show how to track $\ell_t$
along particle trajectories using only score-accessible quantities.  The key
observation is that $\ell_t$ satisfies a \emph{closed} PDE involving only the
scores and their divergences.

\begin{lemma}[PDE for the log-density ratio]
\label{lem:logratio-pde}
Let $\ell_t(x)=\log(\frac{q_t^2(x)}{q_t^1(x)})$. Under \eqref{eq:base-fp-two},
$\ell_t$ satisfies
\begin{equation}
\partial_t \ell_t
=
\langle f_t,\nabla \ell_t\rangle
+\frac{\sigma_t^2}{2}\Delta \ell_t
-\frac{\sigma_t^2}{2}\Big(
  \nabla\!\cdot(s_t^{2}-s_t^{1})
  + \|s_t^{2}\|^2-\|s_t^{1}\|^2
\Big).
\label{eq:logratio-PDE}
\end{equation}
Equivalently, using $\Delta \ell_t=\nabla\!\cdot(s_t^{2}-s_t^{1})$,
\begin{equation}
\partial_t \ell_t
=
\langle f_t,\nabla \ell_t\rangle
+\frac{\sigma_t^2}{2}\Delta \ell_t
-\frac{\sigma_t^2}{2}\Big(
  \Delta \ell_t + \|s_t^{2}\|^2-\|s_t^{1}\|^2
\Big).
\label{eq:logratio-PDE-simplified}
\end{equation}
\end{lemma}

\begin{proof}
Subtract \eqref{eq:dtlogqi-fr} for $i=1$ from the same identity for $i=2$.
The $\nabla\!\cdot f_t$ terms cancel, and the remaining terms can be written
in the form \eqref{eq:logratio-PDE} using
$\nabla \ell_t=s_t^2-s_t^1$ and $\Delta \ell_t=\nabla\!\cdot(s_t^2-s_t^1)$.
\end{proof}

Let $X_t$ solve the FR-guided SDE in \eqref{eq:wsde-fr-final}. Applying
It\^o's formula to $\ell_t(X_t)$, one obtains an SDE for the running
log-ratio along the particle:
\begin{equation}
\mathrm{d}\ell_t(X_t)
=
\Big(
\partial_t \ell_t + \langle v_{t,\beta}^{\mathrm{FR}},\nabla\ell_t\rangle
+\frac{\sigma_t^2}{2}\Delta\ell_t
\Big)\mathrm{d}t
+\sigma_t\langle \nabla\ell_t,\mathrm{d}W_t\rangle.
\label{eq:ell-Ito-general}
\end{equation}
Using $\nabla\ell_t=s_t^2-s_t^1$ and the PDE \eqref{eq:logratio-PDE}, this drift
can be written entirely in terms of $(f_t,\sigma_t)$, the scores
$s_t^1,s_t^2$, and their divergences $\nabla\!\cdot s_t^1,\nabla\!\cdot s_t^2$
(which are standard to estimate in score models via Hutchinson trace
estimators).  In particular, once $\ell_t(X_t)$ is tracked, the FR weights
$\alpha_{t,\beta}^{i}(X_t)$ are obtained from \eqref{eq:alpha-from-ell}.

\medskip

\noindent\textbf{Practical procedure (conceptual).}
Starting from pure noise at $t=1$:
\begin{enumerate}
\item Initialize particles $X_1^{(k)}\sim \mathcal{N}(0,I)$ and initialize
      auxiliary $\ell_1^{(k)}$ (e.g.\ $0$ if the two models share the same
      terminal noise law).
\item For $t\downarrow 0$ (discretize time), at each particle compute:
      \begin{multline}
      s_t^{1}(X_t^{(k)}),\quad s_t^{2}(X_t^{(k)}),\quad
      \\\alpha_{t,\beta}^{2}(X_t^{(k)})=
      \frac{\beta e^{\ell_t^{(k)}/2}}{(1-\beta)+\beta e^{\ell_t^{(k)}/2}},
      \quad \alpha^1=1-\alpha^2,
      \end{multline}
      then form the guided score
      $s_{t,\beta}^{\mathrm{FR}}=\alpha^1 s^1+\alpha^2 s^2$ and drift
      $v_{t,\beta}^{\mathrm{FR}}=-f_t+\sigma_t^2 s_{t,\beta}^{\mathrm{FR}}$.
\item Evolve the state particle by Euler--Maruyama (reverse-time convention)
      \[
      X_{t-\Delta t}=X_t+v_{t,\beta}^{\mathrm{FR}}(X_t)\,\Delta t
      +\sigma_t\sqrt{\Delta t}\,\xi,\quad \xi\sim\mathcal{N}(0,I).
      \]
\item Evolve the auxiliary log-ratio $\ell_t$ using \eqref{eq:ell-Ito-general}
      with $\nabla\ell_t=s^2-s^1$ and $\Delta\ell_t=\nabla\!\cdot(s^2-s^1)$.
\item Update weights using the FR potential
      \[
      \mathrm{d}w_t
      =
      \bar g_{t,\beta}^{\mathrm{FR}}(X_t)\,\mathrm{d}t,
      \qquad
      g_{t,\beta}^{\mathrm{FR}}(X_t)
      =
      -\frac{\sigma_t^2}{4}\alpha^1\alpha^2\|s^1-s^2\|^2,
      \]
      and resample as needed.
\end{enumerate}
This yields samples from $\pi_{0,\beta}^{\mathrm{FR}}$ in the Feynman--Kac / SMC
limit.

The drift \eqref{eq:v-fr-guided} and the potential \eqref{eq:g-fr-explicit} are
expressed using:
\[
f_t,\ \sigma_t,\ s_t^{1},\ s_t^{2},\
\]
No evaluation of $p_{t,\beta}^{\mathrm{FR}}$ (or its normalization constant) is
required. The only additional ingredient beyond score queries is
$\nabla\!\cdot s_t^{i}$ (i.e.\ $\Delta\log q_t^{i}$), which is standard in
modern diffusion implementations (Hutchinson estimators) and is unavoidable if
one insists on an \emph{exact} ratio-tracking mechanism.

\subsubsection{Comparison to the exponential interpolation (classifier-free guidance + FKC)}

For the geometric average $p_{t,\beta}^{\mathrm{geo}}\propto (q_t^1)^{1-\beta}(q_t^2)^\beta$,
the guided score is \emph{exactly} the constant convex combination
$(1-\beta)s^1+\beta s^2$, and the FKC potential collapses to the simple
closed-form discrepancy
$\frac{\sigma_t^2}{2}\beta(\beta-1)\|s^1-s^2\|^2$ (Prop.~3.1 in \cite{skreta_feynman-kac_2025}).
For Fisher--Rao/Hellinger interpolation, the same structure persists:
the corrector is again a negative score-discrepancy penalty, but now it is
\emph{weighted pointwise} by $\alpha^1\alpha^2$, reflecting the state-dependent
uncertainty of the Hellinger mixture.

\end{document}